\documentclass[lettersize, journal]{IEEEtran}
\usepackage{amsmath,amsfonts}
\usepackage{array}
\usepackage[caption=false,font=normalsize,labelfont=sf,textfont=sf]{subfig}
\usepackage{textcomp}
\usepackage{stfloats}
\usepackage{url}
\usepackage{verbatim}
\usepackage{graphicx}
\usepackage{cite}
\usepackage{lmodern}
\usepackage{amssymb}
\usepackage{amsmath}
\usepackage{bm}
\usepackage[linesnumbered,ruled,lined]{algorithm2e} 
\usepackage{algorithmic}
\usepackage{times}
\usepackage{tabularx}
\usepackage{booktabs}
\usepackage[english]{babel}
\usepackage{blindtext}
\usepackage{xcolor}
\usepackage[colorlinks=true, linkcolor=blue, citecolor=blue, urlcolor=blue]{hyperref}
\usepackage{multirow}
\usepackage{makecell}

\newtheorem{proof}{Proof}
\newtheorem{theorem}{Theorem}

\newtheorem{assumption}{Assumption}

\begin{document}
\title{Task-Oriented Formation Decision via Reinforcement Learning: Herding an Attacking Swarm}

\author{Zhaozong Wang, Guibin Sun, Jinyong Chen, and Rui Zhou
\thanks{This work was supported in part by the STI 2030-Major Projects under Grant 2022ZD0208804, in part by the National Natural Science Foundation of China under Grants 62473017 and 62503028. \emph{(Corresponding author: Rui Zhou.)}}
\thanks{The authors are with the School of Automation Science and Electrical Engineering, Beihang University, Beijing 100191, China (e-mail: wangzz0925@buaa.edu.cn; sunguibinx@buaa.edu.cn;  chenjinyong@buaa.edu.cn; zhr@buaa.edu.cn).}}

\markboth{IEEE Transactions on Automation Science and Engineering}%
{Shell \MakeLowercase{{Wang et al.}}: Formation Decision Using Reinforcement Learning for Herding an Adversarial Swarm}

\maketitle

\begin{abstract}
    Multi-robot systems can accomplish tasks that are difficult for a single robot by organizing into task-specific formations. 
    Different from existing studies on multi-robot shape formation, we here study the task-oriented formation decision problem, with a focus on the herding task. 
    This task is challenging due to the attackers' superior maneuverability and their unknown strategies. 
    To address these challenges, we propose the following novel results. 
    First, we encode the formation shape using a low-dimensional parameter vector. 
    This parametric representation reformulates the formation decision as a parameter optimization problem, thereby resolving the limited flexibility of predefined shapes. 
    By optimizing these formation parameters, the defenders' maneuverability disadvantage is mitigated through a formation shape that continuously adapts to task requirements. 
    Second, we develop a reinforcement learning-based policy to regulate the formation parameters. 
    Trained offline in simulations covering diverse attacking strategies, the learned policy can effectively handle adversarial unpredictability during online deployment. 
    Comparative simulations against three baselines demonstrate that our method can successfully accomplish challenging herding tasks.
    Additional scalability simulations further verify its applicability to simulated scenarios involving dozens of robots. 
    We also validate the practical feasibility of our method on a physical robotic platform with 3 attackers and 7 defenders.
    \footnote{\label{fn:pdta}\url{https://github.com/WZZ-Wangzhaozong/formation_decision_herding}}
\end{abstract}

\def\abstractname{Note to Practitioners}
\begin{abstract}
In many applications, it is necessary for a robot swarm to form a task-specific formation, such as herding attackers away from critical facilities. 
However, most existing herding methods typically rely on the superior maneuverability of robots or explicit dynamics of the target, which may not hold in adversarial settings. 
To address these limitations, this paper proposes a formation decision method that coordinates varying numbers of robots. 
The defenders execute this policy through a two-level hierarchical framework. 
First, at the formation level, an appropriate formation shape is determined in real time, and each defender is assigned a desired position. 
Second, at the robot level, each defender computes its control law to reach its assigned position while ensuring collision avoidance. 
In addition, a negotiation protocol is adopted to ensure the applicability of our method in distributed systems. 
More details of our method and experimental results are presented in this paper.
\end{abstract}

\begin{IEEEkeywords}
    Task-oriented formation decision, herding problem, reinforcement learning. 
\end{IEEEkeywords}

\begin{figure}[t]
    \centerline{\includegraphics[width=19pc]{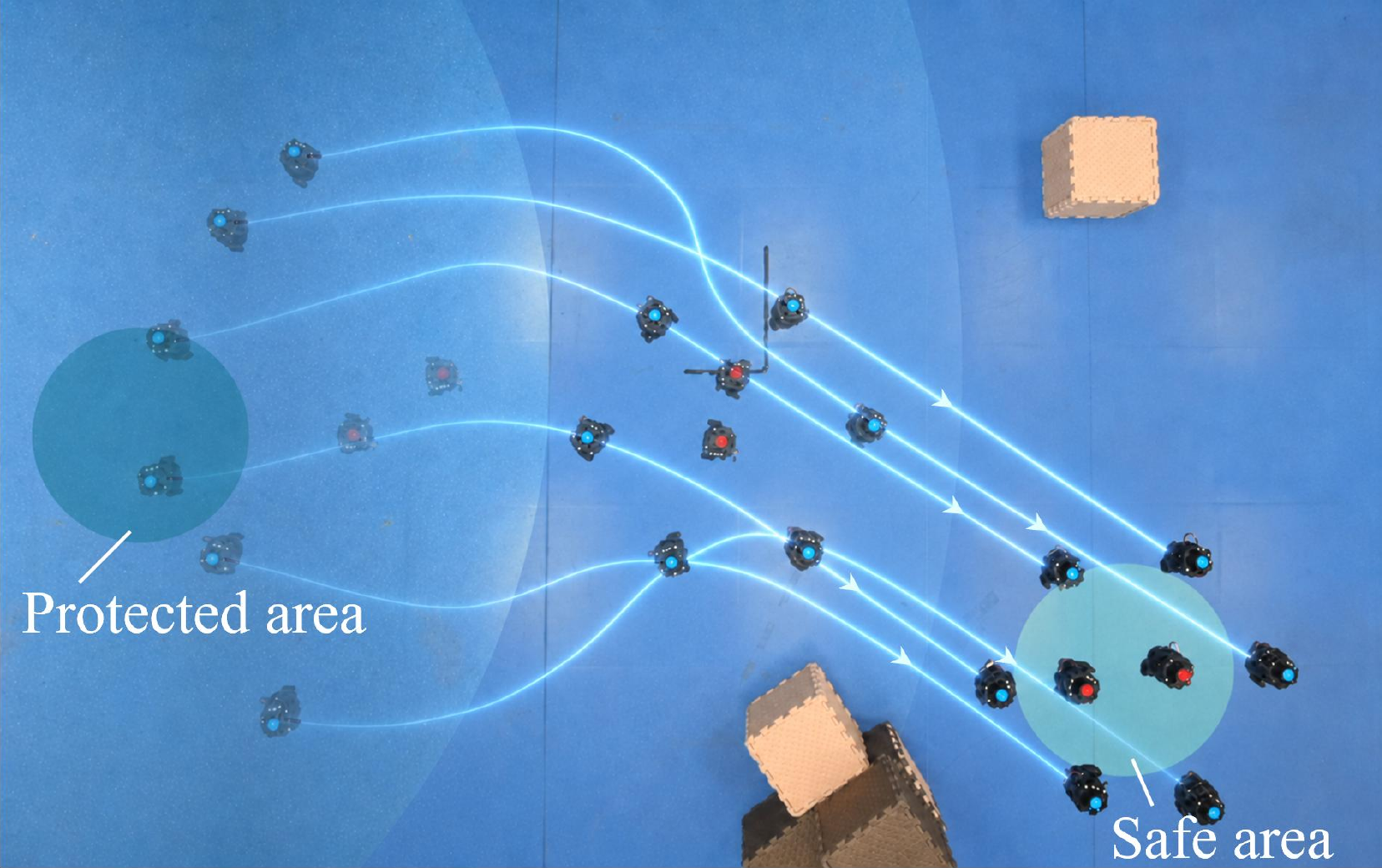}}
    \caption{A group of blue defenders herding two red attackers. 
    Robots with higher transparency correspond to earlier time steps.}
    \label{Fig0}
\end{figure}

\section{Introduction}
\IEEEPARstart{T}{HE} advancement in the theory of multi-robot systems (MRSs) has brought various applications~\cite{sun2023mean}, such as cargo transportation~\cite{Thomas2020dilevery} and area monitoring~\cite{ZHENG2025112216, TAN2026TIE}. 
In particular, multi-robot shape formation has been widely studied, as it effectively enhances collaboration efficiency~\cite{sun2025IIEM}. 
Recent studies have further explored formation switching and reconfiguration to enhance the capabilities of MRSs in complex tasks~\cite{Hu_TCNS, Xing2024UAV}.
For example, in the herding task, a defensive formation can close in around the attackers to encircle them. 
However, while most existing studies still focus on how to drive robots to form a desired shape, the problem of determining the shape for specific tasks has rarely been investigated. 
Specifically, it remains a critical challenge to dynamically adjust the defensive formation to surround highly maneuverable attackers. 

Therefore, this paper investigates the formation decision problem for herding an attacking swarm. 
Existing multi-robot collaboration methods for this task can be categorized into two primary classes. 
The first class comprises distributed motion planning methods~\cite{deng2020multi, chen2024multipleherding, Unmannedsystem1}. 
Rather than strictly enforcing geometric formations, these methods generate control policies at the robot level, where each robot independently herds the target toward the destination. 
The second class consists of formation-based methods~\cite{pierson2017controlling, chipade2021multiagent}.
In these methods, the MRS switches among a set of predefined formations to ensure target containment and subsequently plans trajectories to transport the contained target. 

Despite these efforts, two challenges remain open for further investigation, as outlined below. 
The first challenge is to address the defenders' maneuverability disadvantage relative to the attackers. 
{Rather than} enhancing their maneuverability, a more practical strategy is to organize defenders into a formation specifically designed for herding~\cite{pierson2017controlling, chipade2021multiagent}. 
For example, an arc-shaped formation was proposed in~\cite{chipade2021multiagent} to act as a barrier between the attackers and the protected area. 
This method divides the entire herding process into three separate stages and adopts a predefined formation in each stage. 
As a result, the methods in~\cite{pierson2017controlling, chipade2021multiagent} constrain the defensive formation to predefined shapes rather than allowing it to adapt, which inherently limits its responsiveness to evolving task requirements. 
Therefore, although these methods alleviate the demands on defenders, they fundamentally rely on superior maneuverability to offset the limited formation flexibility. 

The second challenge is how to adjust the formation shape when the attackers' maneuvering strategies are unpredictable. 
In contrast to the aforementioned stage-based switching strategies in~\cite{pierson2017controlling, chipade2021multiagent}, a recent study on formation decision establishes a mapping between task requirements and formation shapes~\cite{Li2025DEFORM}. 
Subsequently, the MRS selects the most appropriate shape from a predefined formation library. 
However, the method in~\cite{Li2025DEFORM} exhaustively identifies only four shapes for an MRS comprising four robots, which limits its scalability to larger groups. 
In addition, to improve decision foresight, most existing studies explicitly model the motion of the target and utilize this to formulate policies for the robots~\cite{chen2024multipleherding, song2021herding}. 
However, the strategies of attackers are typically unknown, making it difficult to predict their future behaviors. 
As a result, enabling the learned policy to respond effectively to unknown attacking strategies during online deployment remains a key challenge. 
Recent studies have increasingly adopted learning-based methods to address environmental uncertainty in challenging applications, such as multi-robot cooperation in dense obstacle environments~\cite{GuWang_UAV}.
To address these challenges, this paper makes the following contributions: 
\begin{itemize}
    \item[1)] We encode the formation shape using five parameters. 
    {Unlike} the methods that linearly transform the formation shape using stress matrices~\cite{Zhao2018Affine, ZHANG2025111935}, our low-dimensional parametric representation enables the formation to smoothly switch between open and closed configurations. 
    This representation transforms the defenders' formation decision into a parameter optimization problem. 
    By optimizing these five parameters, the defensive formation can progressively enhance its surrounding advantage over the attackers while preserving a barrier against them. 
    As a result, the defenders' maneuverability disadvantage can be mitigated by a formation shape that continuously adapts to task requirements. 
    \item[2)] We develop a reinforcement learning (RL)-based policy that regulates the five formation parameters. 
    This design liberates the policy to focus exclusively on high-level formation decisions, as the physical constraints are offloaded to the low-level robot controllers. 
    Specifically, we first formulate a state-dependent reward function to adaptively prioritize different herding objectives during policy optimization. 
    Second, beyond the spatial relationships between the two sides, the observable velocities of the attackers are incorporated into the state space to capture their strategies. 
    With these designs, the policy is trained offline in simulations covering a wide range of attacking behaviors. 
    Therefore, the learned policy can respond to unseen strategies during online deployment. 
    \item[3)] To validate our method, we benchmark it against three baselines that directly control individual robots.
    Comparative simulations highlight the superior performance of our method in addressing the challenging herding task. 
    Furthermore, numerical simulations show that our method can scale to scenarios involving dozens of robots. 
    Finally, real-world experiments using 10 robots demonstrate the practical feasibility of the proposed method.
\end{itemize}

The remainder of this paper is organized as follows. 
Section~\ref{sec_review} reviews related work. 
Section~\ref{sec_preliminary} formulates the problem, and Section~\ref{sec_formation_controller} presents the proposed method. 
Sections~\ref{sec:simulations} and~\ref{sec:experiments} provide simulation and experimental results, respectively. 
Section~\ref{sec:conclusion} concludes this paper.

\section{Related Work}\label{sec_review}
This section provides an overview of related work, including research on herding problems, formation representation, formation control, and learning-based robotic decision-making.

\subsection{Research on Herding Problems}
Herding problems have been extensively explored in the literature, with herding typically referring to guiding a target toward a desired destination. 
A closely related line of research is the pursuit-evasion (PE) game, in which robots are tasked with tracking and eliminating evasive targets~\cite{2021Feedback, Liu2025Observer}. 
These studies have developed effective pursuit policies using either traditional control methods~\cite{2021Feedback} or RL algorithms~\cite{Liu2025Observer}. 
Different from the PE game, the herding problem prioritizes safeguarding a protected area over actively pursuing the target. 
To satisfy this requirement, the PE framework is extended in~\cite{deng2020multi}, where only one robot is deployed to guard the protected area, while the others continue to pursue the target. 

Although the method proposed in~\cite{deng2020multi} is elegantly designed, it assumes that the target actively limits its motion to areas outside the robot's occupied region. 
A more practical solution is to compress the targets' reachable area and eliminate all feasible intrusion paths~\cite{pierson2017controlling, chen2024multipleherding}. 
However, these studies still rely on assumptions that are often unreliable in adversarial settings. 
For instance, the work in~\cite{pierson2017controlling} assumes that the maneuvering strategy of the target is known. 
In addition, the method in~\cite{chen2024multipleherding} assumes the target remains stationary unless subjected to repulsion by robots. 
Some studies still require the robots to possess superior maneuverability~\cite{song2021herding, deng2020multi}. 
Few existing studies have addressed this problem through formation decisions.
A notable exception is the ``StringNet Herding'' method~\cite{chipade2021multiagent}, which coordinates robots into a geometric formation that switches among a set of predefined shapes. 
However, this method limits the responsiveness of the MRS due to the formation constraints. 
To this end, this paper aims to flexibly fine-tune the formation and optimize it in real time. 

\subsection{Formation Representation}
In this paper, formation decision refers to determining the appropriate shape for the MRS, for which formation representation provides the foundation. 

Existing studies commonly employ three classes of formation representation methods. 
The first class, which is also the most intuitive, defines a shape using a number of points equal to that of robots. 
A more concise variation is to designate a small subset of robots as the vertices of the shape~\cite{dong2018theory, yang2023self}. 
However, the representational fidelity of these methods depends heavily on the number of points used. 
Because their dimensionality scales with the number of robots, these point-based methods are unsuitable for subsequent decision-making in large-scale MRSs.

The second class, commonly referred to as density control, models the spatial distribution of robots using mathematical functions~\cite{SINIGAGLIA2025112218, Lei2024CSL}. 
These methods generate a velocity field based on the error between the target and actual density functions, and each robot is driven by this field. 
Nevertheless, the mathematical forms of these functions are highly nonlinear, making them intractable to use directly as variables for formation decision. 

The third class of methods employs a fixed set of parameters to manage the formation shape. 
A representative approach is affine formation control~\cite{Zhao2018Affine}, which scales, rotates, and shears the formation via an affine transformation matrix. 
More recently, the authors in~\cite{ZHANG2025111935} extend the affine formation control by incorporating additional adjustable parameters, thereby enabling a broader range of feasible formation shapes to be explored. 
One limitation of these matrix-based methods is that the transformed formation still depends on its initial configuration, rather than being uniquely determined by the adjustable parameters. 
For specific tasks, a few existing studies parameterize the explicit features of the formation shape~\cite{yang2022autonomous, huang2024cooperative}. 
For example, the method in~\cite{yang2022autonomous} manipulates the horizontal-to-vertical ratio of a ribbon-like swarm to navigate through narrow channels. 

In summary, this paper aims to develop a concise and efficient parametric representation tailored to the herding task. 

\subsection{Formation Control}
Formation control and formation decision are closely related but fundamentally distinct problems. 
The primary objective of formation control is to drive robots to achieve and maintain a given geometric pattern. 
A recent review systematically summarizes representative modeling and control techniques for UAV systems, which provide an important foundation for reliable vehicle and formation control~\cite{Gedefaw2025UAVReview}. 
Recent studies have further improved the safety, robustness, and fault tolerance of formation execution. 
For example, range-only bearing estimation and control barrier functions have been combined for safety-critical formation control~\cite{Chen2025IoTJ}. 
At the vehicle level, robust adaptive controllers have also been developed to ensure safe attitude control of aerial vehicles under faults and uncertainties~\cite{Fu2025AST}. 
These studies improve the reliability of formation execution but generally assume that the desired formation geometry or reference trajectory has already been specified.

Representative formation control methods include leader--follower structures~\cite{TAN2025TASE, Li2025DEFORM}, in which selected leaders perform global planning while followers maintain prescribed relative relationships. 
Some studies further adjust specific formation characteristics according to environmental conditions~\cite{quan2023robust, Liu2024Graph, yang2022autonomous, huang2024cooperative}. 
For example, the formation size can be adjusted according to passage width to ensure safe traversal~\cite{quan2023robust}, while the formation orientation can be rotated to align inter-robot gaps with approaching obstacles~\cite{Liu2024Graph}. 
However, these methods generally modify predefined geometric variables through predetermined rules, rather than jointly {determining them} through a formation decision policy. 

The formation decision considered in this paper is beyond such predefined geometric adjustments. 
It treats the overall formation geometry as an independent high-level decision variable and determines multiple coupled formation parameters according to the evolving task requirements. 
Although learning-based methods have enhanced the autonomous capabilities of individual robots~\cite{TAN2026TIE} and facilitated multi-robot cooperation in complex environments~\cite{Gu2026TCCN}, they primarily learn robot-level actions rather than formation-level decisions. 
In addition, few studies have formulated the formation decision problem for adversarial multi-robot herding tasks. 
To address this gap, this paper represents the defensive formation using a low-dimensional parameter vector and develops a policy to update the formation parameters.

\subsection{Learning-Based Robotic Decision-Making}
Learning-based methods have recently shown strong potential in complex robotic decision-making problems. 
In~\cite{Yuan2023FITEE}, Transformer-based reinforcement learning methods for sequential decision-making are systematically reviewed. 
To improve learning efficiency, a sample-efficient backtrack temporal-difference deep reinforcement learning method is developed in~\cite{Liu2025KBS}. 
Moreover, reinforcement learning from human feedback is further investigated through cooperative policy-reward optimization for LLMs~\cite{Liu2026ESWA}.

Learning-based methods have also been applied to multi-robot and constrained control scenarios. 
For example, the authors in~\cite{GuWang_UAV} propose a constrained reinforcement learning approach for cooperative control of multi-UAV systems in dense obstacle environments. 
In addition, Bayesian optimization is used to develop an online escape strategy for planetary rovers under entrapment conditions~\cite{Guo2024JFR}. 
Perception-aware decision-making has also been investigated in robotic manipulation. 
The authors in~\cite{Li2025TASE} propose an occlusion-aware grasping method that explicitly considers inter-object occlusion relationships and estimates feasible grasping poses in complex multi-object scenes. 
Although these studies demonstrate the effectiveness of learning-based decision-making, they mainly focus on actions, trajectories, control inputs, or policy-reward optimization. 
They do not directly address how to adaptively determine formation-level geometric parameters for adversarial multi-robot herding. 
To bridge this gap, this paper formulates herding-oriented formation decision as a low-dimensional parameter optimization problem and learns a policy to update the formation parameters in real time.

\begin{table}[t]
\centering
\caption{Notations and Their Meanings.}
\label{table:notations_herding}
\begin{tabular}{ll}
\hline
\textbf{Notation}& \textbf{Description} \\ \hline
\multicolumn{2}{l}{\textit{System Sets and Variables:}} \\
$\mathcal{V}_{\rm d}, \mathcal{V}_{\rm a}$ & Sets of defenders and attackers, respectively \\
$\mathcal{O}, \mathcal{P}, \mathcal{S}$ & Obstacles, protected area, and safe area  \\
$N_{\rm d}, N_{\rm a}$ & Number of defenders and attackers \\
$\bm{p}, \bm{v}, \bm{u}$ & Position, velocity, and acceleration vectors of a robot  \\
$v_{\rm {max}}, u_{\rm {max}}$ & Maximum velocity and acceleration constraints  \\
$\rho_{\rm p}, \rho_{\rm s}$ & Radius of the protected area and safe area  \\ \hline
\multicolumn{2}{l}{\textit{Formation Parameters ($\bm \theta$):}} \\
$\bm{p}_{\rm c}$ & Translation parameter relative to the protected area \\
$\varphi$ & Parameter defining formation orientation  \\
$\zeta$ & Parameter for inter-robot distance  \\
$\beta$ & Parameter defining formation opening  \\
$\hat{\bm \theta}_i$ & Defender $i$'s local estimate of formation parameters  \\ \hline
\multicolumn{2}{l}{\textit{Reinforcement Learning (RL) Components:}} \\
$\bm s, \bm a, r$ & State space, action space, and reward function \\
$r_{\rm att}$ & Radius of the attacking swarm\\
$\bm {v}_{\rm att}, \varphi_{{\bm v}_{\rm att}}$ & Average velocity and direction of the attacking swarm \\
$\varphi_{\rm nea}$ & Non-escape angle quantifying surrounding advantage  \\
$r_{\rm ter}, f(\cdot)$ & Terminal and immediate reward terms \\
$\gamma$ & Discount factor in the DDPG framework  \\ \hline
\multicolumn{2}{l}{\textit{Control and Negotiation Parameters:}} \\
$\bm{p}_i^*$ & Desired position for defender $i$  \\
$\bm{u}_i^{\rm c}, \bm{u}_i^{\rm f}, \bm{u}_i^{\rm a}$ & Control terms: connectivity preservation, formation \\ & maintenance, and collision avoidance \\
$\mathcal{G}, \mathcal{E}, L(\mathcal{G})$ & Communication graph, edge set, and Laplacian matrix \\
$c_\theta$ & Negotiation gain for parameter consensus \\
$\lambda_2(\mathcal{G})$ & Algebraic connectivity (second-smallest eigenvalue) \\ \hline
\end{tabular}
\end{table}

\section{Preliminaries and Problem Formulation}\label{sec_preliminary}
This section first introduces the considered herding scenario and the capabilities of both attackers and defenders. 
Second, we formulate the herding problem. 

\subsection{Preliminaries}
We consider a scenario involving defenders, attackers, obstacles, and protected and safe areas, as shown in Fig.~\ref{Fig1_defender_phase}. 
This scenario is confined to a square region spanning $[-{r_{\rm env}}, {r_{\rm env}}]^2$.
The set of defenders is denoted by $\mathcal{V}_{\rm{d}}=\{1, 2, \dots, N_{\rm{d}}\}$, where $N_{\rm{d}}$ is the number of defenders. 
Similarly, the sets of attackers and obstacles are denoted by $\mathcal{V}_{\rm{a}}=\{a_1,a_2, \dots, a_{N_{\rm{a}}}\}$ and $\mathcal{O}=\{o_1, o_2, \dots, o_{N_{\rm{o}}} \}$, respectively. 
In practical scenarios, critical facilities may have irregular shapes. 
For generality, both the protected and safe areas are approximated by minimum bounding circles, defined as $\mathcal{P}=\{\bm{p}\in\mathbb{R}^2:\Vert \bm{p}-\bm{p}_{\rm p}\Vert\leq\rho_{\rm p}\}$ and $\mathcal{S}=\{\bm{p}\in\mathbb{R}^2:\Vert \bm{p}-\bm{p}_{\rm s}\Vert\leq\rho_{\rm s}\}$, where $(\bm{p}_{\rm p}, \rho_{\rm p})$ and $(\bm{p}_{\rm s}, \rho_{\rm s})$ are their centers and radii, respectively. 
The sensing, maneuvering, and communication capabilities of both sides are described as follows.

\subsubsection{Sensing Ability}
The positions of the protected area and the defenders are assumed to be known to the attackers throughout their intrusion process. 
On the defender side, a navigation system is deployed in $\mathcal{P}$ to detect the positions and velocities of attackers within the area $\mathcal{Z}_{\rm d}^{\rm sen} = \{\bm{p}\in \mathbb{R}^2:\Vert \bm{p} - \bm{p}_{\rm p}\Vert \leq r_{\rm env}\}$ and send this information to the defenders. 
To reflect real-world sensor limitations, the detected information is modeled by incorporating distance-dependent Gaussian noise. 
Specifically, the observations satisfy $\widetilde{\bm p} \sim \mathcal{N}(\bm{p}, \sigma_p^2(r) \bm{I})$ and $\widetilde{\bm v} \sim \mathcal{N}(\bm{v}, \sigma_v^2(r) \bm{I})$, where $r = \Vert \bm p - \bm{p}_{\rm p}\Vert$ denotes the relative distance to the navigation system. The distance-dependent standard deviation $\sigma(r)$ is modeled as a piecewise function:
\begin{equation*}
\sigma(r) = \begin{cases}
0.0, & r \leq r_1^{\rm sen} \\
\sigma_{v/p} \cdot \frac{r - r_1^{\rm sen}}{r_2^{\rm sen} - r_1^{\rm sen}}, & r_1^{\rm sen} < r \leq r_2^{\rm sen} \\
\sigma_{v/p}, & r > r_2^{\rm sen}
\end{cases}
\end{equation*}
where $\sigma_{v/p}$ denotes the maximum standard deviation. 
The thresholds $r_{\rm 1}^{\rm sen}$ and $r_{\rm 2}^{\rm sen}$ are set to $\frac{r_{\rm env}}{3}$ and $\frac{2r_{\rm env}}{3}$ in this paper. 

\subsubsection{Maneuverability}
The robots on both sides follow double-integrator dynamics, $\dot{\bm{p}} = \bm{v}$ and $\dot{\bm{v}} = \bm{u}$, where $\bm{u}=[u_x,u_y]^\top$ denotes the acceleration input. 
The acceleration of each robot is constrained by $\Vert \bm{u}\Vert \leq u_{\rm max}$.
In addition, the velocities are subject to the constraint $\Vert \bm{v}\Vert \leq v_{{\rm max}}$. 
The attackers exhibit superior maneuverability, characterized by {$v^{\rm a}_{\rm max} > v^{\rm d}_{\rm max}$} and {$u^{\rm a}_{\rm max} > u^{\rm d}_{\rm max}$}, where the subscripts $\rm a$ and $\rm d$ denote the attackers and defenders, respectively. 

\subsubsection{Communication Ability}
The attackers can communicate with each other regardless of distance. 
Conversely, defenders are restricted to communicating solely with their immediate neighbors under a distributed architecture. 
\begin{figure}[t]
    \centerline{\includegraphics[width=21pc]{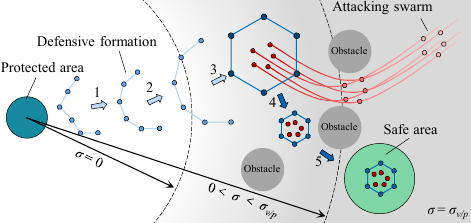}}
    \caption{Herding process. The process consists of two phases: the defending phase comprising Steps 1--3 and the guiding phase comprising Steps 4--5. Step 1: Defenders gather to form an open-arc formation. Step 2: The formation scales to enhance its surrounding capability. Step 3: The formation surrounds the attacking swarm. Step 4: The closed-loop formation transports the attackers. Step 5: The formation reaches the safe area $\mathcal{S}$.}
    \label{Fig1_defender_phase}
\end{figure}

\subsection{Problem Formulation}
This paper aims to use defenders to herd an attacking swarm into a designated safe area before the swarm enters the protected area. 
As depicted in Fig.~\ref{Fig1_defender_phase}, the herding process is divided into the defending and guiding phases. 
During the defending phase, the defensive formation adapts its configuration to eventually surround the attackers. 
Once the attackers are surrounded, the process transitions into the guiding phase to transport all attackers to the safe area. 
In the following, we present two assumptions required for our method.

\begin{assumption}\label{Ass_1}
    We assume that adjacent defenders create an impassable barrier for the attackers if the distance between the adjacent defenders is less than a given threshold $r_{\rm d}^{\rm str}$. 
\end{assumption}
\begin{assumption}\label{Ass_3}
    Consistent with~\cite{chipade2021multiagent,chen2024multipleherding}, we assume that a navigation system is deployed in $\mathcal{P}$ to perceive all the attackers and broadcast this information to all defenders. 
\end{assumption}

In the herding scenarios, the attackers perform risk-aware navigation, enabling them to proactively avoid approaching defenders while advancing toward the protected area. 
The maneuvering strategy of the attackers is given by
\begin{equation}
    \begin{split}
        \bm{u}_{a_j} =& k_{\rm ad} g(\Vert \bm{r}_{a_j}^{\rm str} \Vert,r_{\rm ad}^{\rm safe},r_{\rm ad}^{\rm avo}) + k_{\rm ap}\dfrac{\bm{p}_{a_j}-\bm{p}_{\rm p}}{\Vert \bm{p}_{a_j}-         \bm{p}_{\rm p}\Vert}  \\
        &+ k_{\rm aa} g(\Vert \bm{p}_{a_j}-\bm{p}_{a_j^{'}}\Vert,r_{\rm aa}^{\rm safe},r_{\rm aa}^{\rm avo}) 
    \end{split}
    \label{Equ_attacker_strategy}
\end{equation}
where the constants $k_{\rm ad}$, $k_{\rm ap}$, and $k_{\rm aa}$ represent the weighting factors for avoiding defenders, entering the protected area, and avoiding internal collisions, respectively; $r_{\rm aa}^{\rm avo}$, $r_{\rm aa}^{\rm safe}$, $r_{\rm ad}^{\rm avo}$, and $r_{\rm ad}^{\rm safe}$ define the distance thresholds for activating the corresponding potential-field terms. 
Moreover, $a_j^{'}$ denotes the attacker closest to $a_j$, and $\Vert \bm{r}_{a_j}^{\rm str} \Vert$ is the minimum distance between $a_j$ and the defensive formation. 
The potential-field function $g$ is defined in Eq.~\eqref{Equ_weight}:
\begin{equation}
    \begingroup
    g(r,r_1,r_2) = \begin{cases}
        \frac{1}{r-r_1} \left[1+\cos\left(\pi\frac{r-r_1}{r_2-r_1}\right)\right], & r_1< r\leq r_2 \vspace{3pt} \\
        0,                                                                      & r>r_2
    \end{cases}
    \endgroup
    \label{Equ_weight}
\end{equation}
where $r_1$ and $r_2$ are given constants used to adjust the shape of $g$. 
Although the attacking strategy is fixed in the form of {Eq.~\eqref{Equ_attacker_strategy}}, we consider multiple combinations of these hyperparameters to enrich the diversity of attacking behaviors.
In addition, the parameters $k_{\rm ad}$, $k_{\rm ap}$, $k_{\rm aa}$, $r_{\rm aa}^{\rm avo}$, $r_{\rm aa}^{\rm safe}$, $r_{\rm ad}^{\rm avo}$, and $r_{\rm ad}^{\rm safe}$ remain unknown to the defenders, rendering the attackers' behaviors unpredictable. 
This setting requires our method to operate solely on observable information about the attackers.

\section{Formation Decision Method for Herding Tasks}\label{sec_formation_controller}
In this section, we first introduce the proposed parametric representation of the defensive formation. 
Second, we develop an RL-based herding policy that regulates the formation parameters in response to task requirements. 

\subsection{Parametric Representation for Defenders}
The configuration of the defensive formation, including its position, shape, and size, is represented by the vector 
\begin{equation*}
    \bm{\theta} = \begin{bmatrix}
        \bm{p}_{\rm c}^{\top} & \varphi & \zeta & \beta
    \end{bmatrix}^{\top}
\end{equation*}
where $\bm{p}_{\rm c} \in \mathbb{R}^2$ represents the translation, while $\varphi$, $\zeta$, and $\beta$ are the parameters defining orientation, scaling, and opening, respectively. 
A visualization of this representation is shown in Fig.~\ref{Fig2_parametric_formation}. 
It should be noted that $\bm{p}_{\rm c}$ is defined relative to the center of the protected area. 
The policy regulating the above parameters during the defending phase is learned using an RL algorithm, as detailed in the next subsection.
\begin{figure}[h]
    \centerline{\includegraphics[width=21.5pc]{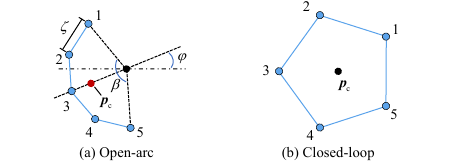}}
    \caption{Illustration of open-arc and closed-loop formations.}
    \label{Fig2_parametric_formation}
\end{figure}

Once all the attackers are surrounded by defenders, the defending phase transitions to the guiding phase. 
As long as the attackers remain contained within the closed-loop formation, there always exists a way to transport them to the safe area. 
To ensure this, we conservatively eliminate the possibility of attacker escape caused by deformation of the defensive formation.
Accordingly, the parameters $\varphi$, $\zeta$, and $\beta$ are kept constant, and the formation only needs to adjust $\bm{p}_{\rm c}$. 
In this setting, the closed-loop formation is treated as a rigid entity with a fixed radius, and the classical velocity obstacle method proposed in~\cite{fiorini1998motion} is employed to adjust ${\dot{\bm{p}}_{\rm c}}$, as detailed below: 
\begin{equation*}
    {\dot{\bm{p}}_{\rm c}} = \mathop{\rm argmin}\limits_{\boldsymbol{v}\notin \mathop{\cup}\limits_{j=1}^{N_{\rm o}}\mathcal{VO}_j} \Vert \boldsymbol{v}-{{\bm{v}}}_{{\rm ref}} \Vert
\end{equation*}
with 
\begin{equation*}
    {{\bm{v}}}_{{\rm ref}} = \frac{\bm{\bar p}_{\rm s}-{\bm{p}}_{{\rm c}}}{\left\lVert  \bm{\bar p}_{\rm s}-{\bm{p}}_{{\rm c}}\right\rVert } \cdot \min \big({v_{\rm max}^{\rm{a}}, v_{\rm max}^{\rm{d}}}\big)
\end{equation*}
where ${\bm{\bar p}}_{{\rm s}} = {\bm{p}}_{{\rm s}}-{\bm{p}}_{{\rm p}}$ and $\mathcal{VO}_j$ encompasses all possible velocities that could result in a collision with the obstacle $o_j$ within a look-ahead time $\Delta t$. 
Notably, the term $ \min \big({v_{\rm max}^{\rm{a}}, v_{\rm max}^{\rm{d}}}\big)$ ensures that the collective motion remains feasible even in scenarios where {$v_{\rm max}^{\rm{a}}< v_{\rm max}^{\rm{d}}$}. 

After the parameters and their update policy are determined, the next step is to control the defenders to realize the desired formation. 
The robot-level control law, which accounts for formation maintenance, connectivity preservation, and collision avoidance, is presented in Appendix~\ref{Control law}. 
In addition, formation parameter consensus is a fundamental requirement in distributed multi-robot coordination. Recent studies have shown that the agreement behavior of systems can be significantly affected by the design of coordination mechanisms~\cite{Ma2023Automatica}.
To enable distributed implementation, we introduce the following negotiation protocol for defender $i$:
\begin{equation}
    \dot{\hat{\bm{\theta}}}_i
    = \pi_i(\bm{s}_i)
    + c_{\theta}\sum_{j\in\mathcal{N}_i}
    \big(\hat{\bm{\theta}}_j-\hat{\bm{\theta}}_i\big).
    \label{Equ_negotiation}
\end{equation}

Let $\bm{\theta}^{*}$ denote the ground-truth optimal parameter {vector.} 
Assuming that the policy is well trained, we have
\begin{equation*}
    \big(\hat{\bm{\theta}}_i - \bm{\theta}^{*}\big)^\top \pi_i(\bm{s}_i) \leq 0,
    \label{Equ_well_trained}
\end{equation*}
with equality only when $\hat{\bm{\theta}}_i = \bm{\theta}^{*}$. 
The convergence of the formation parameter estimates is established in Theorem~\ref{Theo4_formation_parameter}. 
\begin{theorem}
\label{Theo4_formation_parameter}
    Under the negotiation protocol in Eq.~\eqref{Equ_negotiation}, all defenders' estimates satisfy 
    \[
        \lim_{t\rightarrow\infty} 
        \hat{\bm{\theta}}_i = \bm{\theta}^{*},\quad \forall i\in\mathcal{V}_{\rm d}.
    \]
\end{theorem}
\begin{proof}
    See Appendix~\ref{proof1}.
\end{proof}

\begin{figure*}[t]
    \centering
    \includegraphics[width=1.0\textwidth]{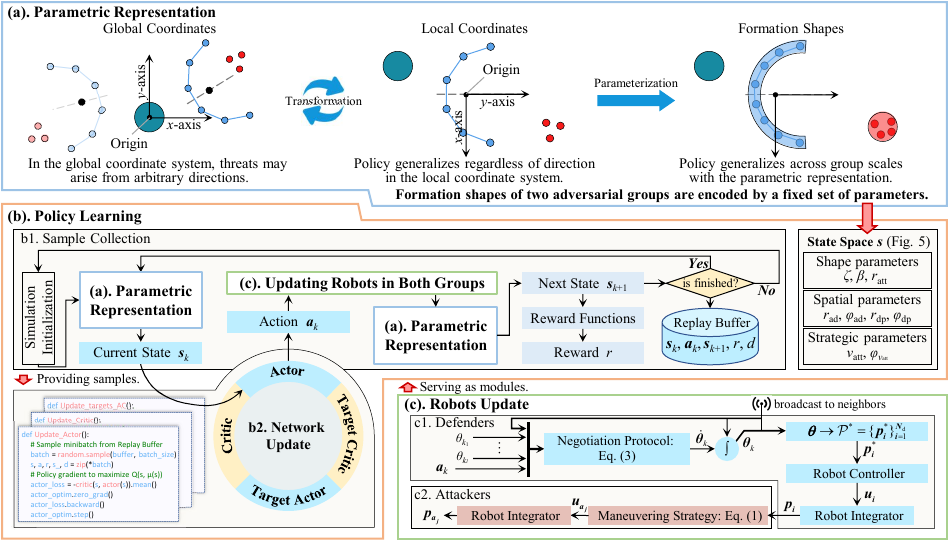}
    \caption{RL-based framework for formation parameter update. 
    (a) Parametric representation of the defensive formation and the attacking swarm. The resulting policy can generalize across different attacking directions and group scales. 
    (b) Learning process of the policy for updating the formation parameters. 
    (c) Robot motions are computed based on their desired positions.}
     \label{Method}
\end{figure*}

\subsection{RL-Based Formation Parameter Decision}
This subsection presents how to derive an RL-based policy that regulates $\bm{\theta}$ during the defending phase. 

\subsubsection{Overall Framework}
The overall framework is illustrated in Fig.~\ref{Method}. 
First, the two adversarial groups are represented as formation shapes with specific geometric characteristics, allowing scenarios with different group scales to be described in a fixed-dimensional state space. 
Second, offline simulations covering diverse attacking strategies are constructed to serve as the policy learning environment.
To implicitly embed the attacking strategy into the policy learning process, the observable velocity information of the attacking swarm is incorporated into the state space. 

Based on the above modeling, the policy is iteratively refined within the RL framework. 
At each time step, the current situation is transformed into the state vector $\bm{s}_{k}$, which is then fed into the actor network to generate an action $\bm{a}_{k}$ for updating the formation. 
Subsequently, the attackers update themselves according to Eq.~\eqref{Equ_attacker_strategy}, and the resulting situation is transformed into the next state $\bm{s}_{k+1}$. 
The reward $r$ is then calculated from $\bm{s}_{k+1}$ based on several reward terms defined below. 
As a result, the tuple $(\bm{s}_{k}, \bm{a}_{k}, \bm{s}_{k+1}, r, d)$ is stored in the replay buffer for subsequent training, where $d$ encodes the episode outcome: $d=1$ for successful herding, $d=2$ for failure, and $d=0$ if the episode continues or times out. 
Finally, the networks, including the actor, critic, and target networks, are updated using the Deep Deterministic Policy Gradient (DDPG) algorithm. 

\subsubsection{State and Action Spaces}
The actor network learns a policy for updating $\bm{\theta}$, and the action space is defined as $\bm{a}=\dot{\bm{\theta}}$. 

The state space should represent varying group sizes, intrusion directions, and attacking strategies. 
First, to enable the policy to generalize across different group sizes, the first component of the state space consists of the shape parameters $\zeta$, $\beta$, and $r_{\rm att}$, where $r_{\rm att}$ denotes the distance from the center of the attacking swarm to its farthest attacker. 
Second, to handle intrusions from arbitrary directions, the spatial relationships among the defensive formation, the attacking swarm, and the protected area are transformed into a coordinate frame in which the formation orientation is aligned with a reference angle of $0^\circ$, as shown in Fig.~\ref{Method}(a). 
The corresponding spatial parameters are $r_{\rm dp}$, $\varphi_{\rm dp}$, $r_{\rm ad}$, and $\varphi_{\rm ad}$, where $r$ and $\varphi$ denote the relative distance and angle between the corresponding entities. 
The subscripts ${\rm a}$, ${\rm d}$, and ${\rm p}$ refer to the attacking swarm, the defensive formation, and the protected area, respectively.
Third, to enable the learned policy to respond to different attacking strategies, the observed velocities of the attackers are incorporated into the state space. 
Accordingly, the third component of the state space is the average velocity of the attacking swarm, denoted by $\bm{v}_{\rm att}$. 

Because the spatial angles may wrap across the $-\pi$ to $\pi$ boundary, directly using them would introduce discontinuities that hinder gradient-based optimization. 
Therefore, each angle is encoded using its sine and cosine. 
As a result, the state space is defined as 
\begin{equation*}
\bm{s} = 
\begin{bmatrix}
\zeta, & \beta, & r_{\rm att}, & r_{\rm dp},  \\
\cos(\varphi_{\rm dp}), & \sin(\varphi_{\rm dp}), & r_{\rm ad}, & \cos(\varphi_{\rm ad}),\\
\sin(\varphi_{\rm ad}), &\Vert\bm{v}_{\rm att}\Vert, & \cos(\varphi_{\bm{v}_{\rm att}}), & \sin(\varphi_{\bm{v}_{\rm att}})
\end{bmatrix}.
\end{equation*}
The state space is graphically illustrated in Fig.~\ref{Fig_alignment_angle}. 
\begin{figure}[h]
    \centerline{\includegraphics[width=18pc]{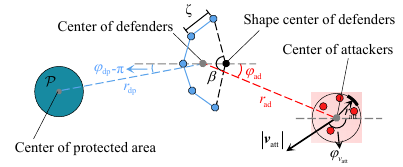}}
    \caption{Schematic diagram of the state space.}
    \label{Fig_alignment_angle}
\end{figure}

\subsubsection{Reward Function}
The reward for each state transition $\bm{s}_k\xrightarrow{\bm{a}_k} \bm{s}_{k+1}$ is defined as $r(\bm{s}_k,\bm{a}_k) = r_{\rm ter} + f(\bm{s}_{k+1})$, 
where $r_{\rm ter}$ and $f(\bm{s}_{k+1})$ are the terminal and immediate rewards, respectively. 
Reward design plays a critical role in policy learning~\cite{zhao2024mathematical}, and the two types of rewards are detailed as follows. 

The primary criterion for evaluating the current policy is its effectiveness in enabling the defensive formation to surround the attacking swarm. 
The terminal reward is defined as $r_{\rm ter} = 2000$ for successful herding, $r_{\rm ter} = -2000$ for failure, and $r_{\rm ter} = 0$ if the episode continues or times out. 
In addition, the following four immediate reward terms are introduced to alleviate reward sparsity. 
It should be noted that these immediate rewards are set to non-positive values. 
This setup penalizes unnecessary time steps, encouraging the formation to complete the surrounding task as quickly as possible and thereby maximize the state-value function. 

First, the defensive formation should surround the attacking swarm as effectively as possible. 
To quantify the surrounding advantage, the concept of the ``non-escape angle'' $\varphi_{\rm nea}  \in (0, 2\pi]$ is introduced.
This angle defines an area that is impassable for all the attackers, which is the gray sector present in Fig.~\ref{Fig_non_escape_angle}. 
Accordingly, the reward $f^{\rm sur}$ is defined as: 
\begin{equation*}
    f^{\rm sur} = \frac{\varphi_{\rm nea}}{2\pi} - 1.0.
\end{equation*}
In addition, when $\varphi_{\rm nea}$ approaches zero, the attacking swarm has nearly bypassed the defensive formation. 
In this case, subsequent surrounding becomes extremely difficult because the attackers possess superior maneuverability. 
To prevent such situations, an additional penalty term, $f_{\rm min}^{\rm sur} = -1.0$, is added to $f^{\rm sur}$ when $\varphi_{\rm nea}<\hat \varphi_{\rm nea}$. 
Here, the lower-bound safety threshold $\hat{\varphi}_{\rm nea}$ is empirically set to $\frac{\pi}{20}$.

Second, the defensive formation needs to prevent the attacking swarm from approaching the protected area. 
However, once the attackers are herded far enough away from the protected area, further herding is unnecessary, and the formation needs to focus on increasing other advantages. 
Therefore, the reward $f^{\rm dis}$ is defined as: 
\begin{equation*}
    f^{\rm dis} = \min\big({r_{\rm ap} - \hat r_{\rm ap}}, 0.0\big) 
\end{equation*}
where $\hat r_{\rm ap}$ is a given distance determined by the size of the herding scenario. 
Here, the value of $\hat r_{\rm ap}$ is set to $\frac{r_{\rm env}}{2}$.
\begin{figure}[b]
    \centerline{\includegraphics[width=21.pc]{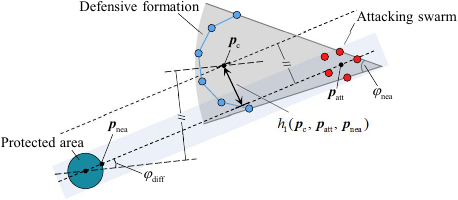}}
    \caption{Schematic diagram of reward functions. The symbol ---//--- indicates that the two lines are parallel.
    In this diagram, the defensive formation needs to move to the lower right while rotating counterclockwise.}
    \label{Fig_non_escape_angle}
\end{figure}

Third, the defensive formation is expected to function as an effective barrier interposed between the protected area and the attacking swarm. 
Specifically, the formation should stay in front of the protected area rather than behind it. 
Considering this spatial constraint, we introduce the concept of the ``nearest defending position'' $\bm{p}_{\rm nea}$, defined as
$\bm{p}_{\rm nea}=\bm{p}_{\rm att}\frac{\rho_{\rm p}}{\Vert\bm{p}_{\rm att}\Vert}$. 
Here, $\bm{p}_{\rm{att}}$ denotes the center of the attacking swarm relative to the protected area, and $\rho_{\rm p}$ is the radius of the protected area. 
Accordingly, the reward $f^{\rm pos}$ is defined as:
\begin{equation*}
f^{\rm pos} = {\rm min}( \rho_{\rm p} - h_1(\bm{p}_{\rm{c}}, \bm{p}_{\rm{att}}, \bm{p}_{\rm{nea}}), 0.0)
\end{equation*}
where $h_1(\bm{p}, \bm{p}_1, \bm{p}_2)$ is used to compute the shortest distance from point $\bm{p}$ to the line segment connecting $\bm{p}_1$ and $\bm{p}_2$. 
This reward imposes a penalty when the formation moves outside the corridor defined by $h_1(\bm{p}_{\rm c},\bm{p}_{\rm att},\bm{p}_{\rm nea})\leq\rho_{\rm p}$. 
As shown in Fig.~\ref{Fig_non_escape_angle}, the formation is required to move into the gray-blue corridor in the lower right and form a defensive barrier. 

Finally, the opening of the defensive formation should always be directed toward the attacking swarm.
This ensures that the formation can quickly respond to the swarm's maneuvers. 
For example, in Fig.~\ref{Fig_non_escape_angle}, the formation needs to rotate counterclockwise to directly face the attacking swarm. 
Consequently, the reward $f^{\rm ali}$ is defined as: 
\begin{equation*}
f^{\rm ali} = {\rm min}(\varphi_{\rm diff}-h_2(\varphi, \varphi_{\rm {ad}}), 0.0)
\end{equation*}
where $h_2(\cdot, \cdot )$ is used to evaluate the difference between two angles, properly handling the $2\pi$-periodicity (e.g., treating $-\pi$ and $\pi$ as identical).
Here, the value of $\varphi_{\rm diff}$ is set to $\frac{\pi}{60}$, which is designed to prevent the formation from adjusting its orientation too frequently for angle alignment.

The four components operate together to drive the defensive formation to complete the surrounding task. 
During the initial herding phase, $f^{\rm pos}$ and $f^{\rm ali}$ regulate the formation position $\bm{p}_{\rm c}$ and orientation $\varphi$ to establish a barrier directly confronting the attacking swarm. 
Building on this posture, $f^{\rm sur}$ adjusts the opening size of the formation to extend the barrier length, while $f^{\rm dis}$ encourages the formation to drive the attackers away from the protected area as they approach. 
However, when the formation is close to surrounding the attackers, further emphasizing $f^{\rm pos}$, $f^{\rm ali}$, and $f^{\rm dis}$ may cause the defenders to waste their already limited maneuverability on repeated posture adjustments, while neglecting the surrounding action.
At this stage, the priority should be increasing the surrounding advantage. 
To resolve this conflict and eliminate the reliance on sensitive hand-tuned weights, we explicitly introduce a state-dependent switching mechanism. 
Once the formation can complete the encirclement by slightly closing its opening, the components $f^{\rm pos}$, $f^{\rm ali}$, and $f^{\rm dis}$ are deactivated. 
As a result, the combined immediate reward $f$ is defined as 
\begin{equation}
\small
\label{reward_function}
f(\bm{s}) = \begin{cases}
    k_{\rm sur}\cdot f^{\rm sur}+k_{\rm dis}\cdot f^{\rm dis}+k_{\rm pos}\cdot f^{\rm pos}\\+k_{\rm ali}\cdot f^{\rm ali},\ &\varphi_{\rm nea} \leq \pi \\
    k_{\rm sur}\cdot f^{\rm sur},\ &\varphi_{\rm nea} > \pi  \\
\end{cases}
\end{equation}
where $k_{\rm sur}$, $k_{\rm dis}$, $k_{\rm pos}$, and $k_{\rm ali}$ represent the weighting coefficients for their respective reward components. 
In Section~\ref{subsec:learning}, we demonstrate that the state-dependent reward design is more effective than its static-weight counterpart in guiding policy learning for the herding task.

\subsubsection{Policy Learning}
Policy learning is conducted in simulations with a fixed number of simulated robots. 
In each step, a sample $(\bm{s}_k, \bm{a}_k, \bm{s}_{k+1}, r, d)$ is stored in the replay buffer $\mathcal{B}$, where $\bm{s}_k$ and $\bm{s}_{k+1}$ denote the current and next states, $\bm{a}_k$ is the action taken, $r$ is the reward computed based on $\bm{s}_{k+1}$, and $d$ indicates the termination state. 
We adopt the DDPG algorithm to train the policy in a continuous action space by leveraging the transition samples stored in the replay buffer $\mathcal{B}$. 
Since the low-dimensional state space eliminates the need to process high-dimensional variables, we employ lightweight fully connected neural networks for policy learning. 
Detailed network architectures are presented in Appendix~\ref{DDPG_framework}. 

\subsubsection{Online Deployment}
We analyze the computational complexity of our method to demonstrate its feasibility for real-time deployment, as detailed in Appendix~\ref{Complexity}. 
Online deployment also requires the learned policy to generalize beyond the fixed group size used during policy learning. 
Its direct application to scenarios with substantially different group sizes is therefore inherently limited. 
Moreover, when the actual effective barrier threshold $\hat{r}_{\rm d}^{\rm str}$ differs from the nominal value $r_{\rm d}^{\rm str}$ assumed during policy learning, Assumption~\ref{Ass_1} may no longer hold. 
To address these two deployment limitations, we introduce a scale-adaptive policy transfer mechanism and a barrier-threshold correction mechanism in Appendices~\ref{policy_transfer} and~\ref{correction}, respectively. 

\section{Simulation Results}\label{sec:simulations}
In this section, we first present the policy learning process. 
Second, we conduct ablation studies to demonstrate the necessity of the velocity observation and the negotiation protocol.
Third, the performance of our method is benchmarked against three baselines. 
Finally, we evaluate the scalability and robustness of our method across varying group sizes. 
\subsection{Simulation Settings} \label{subsec:settings}
\begin{figure*}[t]
    \centering
    \includegraphics[width=0.94\textwidth]{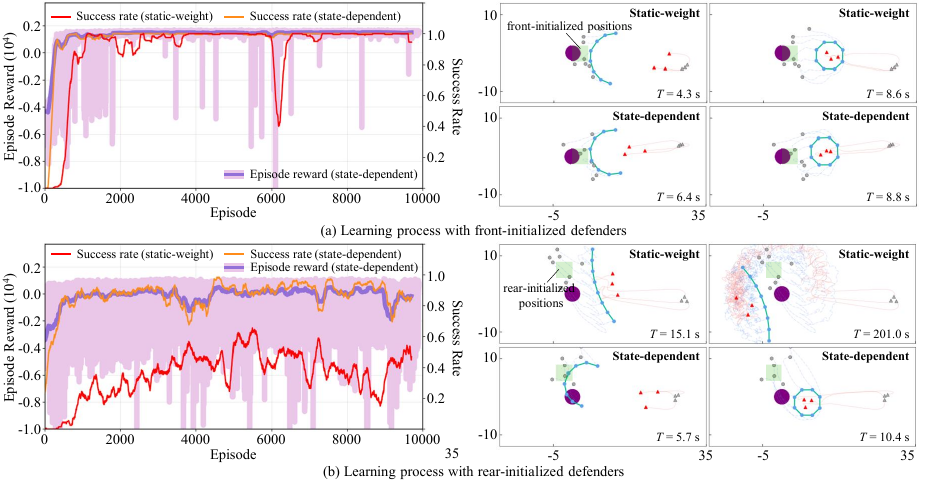}
    \caption{Learning process. 
    For each subfigure, the left panel shows the cumulative reward and success rate during policy learning, while the right panel presents representative simulation snapshots corresponding to the state-dependent and static-weight reward functions. }
    \label{Learning_pro}
\end{figure*}
All simulations are conducted on a workstation with an Intel i7-13700KF CPU and an NVIDIA RTX 4080 GPU, using PyTorch with Python 3.11. 
The training batch size is set to 64, and the main parameters are summarized in Table~\ref{parameter1}. 
It can be observed that the attackers exhibit 25\% higher maneuverability than the defenders. 
The PyTorch implementation is available in the shared repository.\hyperref[fn:pdta]{\textsuperscript{\ref{fn:pdta}}}
To further validate our method, we compare it with three baselines detailed as follows. 

1) Voronoi-based method~\cite{deng2020multi}.
This method assigns defensive roles to a multi-robot team to protect the area from a single attacker. 

2) Satellite-based method~\cite{chen2024multipleherding}.
This method plans paths for defenders and assigns each robot a rotation center, generating satellite-like circular or spiral motions around the attacker or neighbors to compress the attacker's reachable area. 

3) Model predictive control (MPC)-based method~\cite{Chen2024RAL}.
This method addresses multi-pursuer multi-evader scenarios using MPC. 
Based on predicted state sequences, defenders are assigned to specific attackers and driven to form regular polygonal formations for coordinated encirclement.
\begin{table}[!t] 
\centering 
\scriptsize
\setlength{\tabcolsep}{3pt}
\caption{Parameter Settings.}
\label{parameter1} 
\begin{tabular}{c c c c c c} 
\toprule 
Parameter & Value & Parameter & Value & Parameter & Value \\ 
\midrule 

$\rho_{\rm s}$ & $2.0~\mathrm{m}$ 
& $\rho_{\rm p}$ & $2.0~\mathrm{m}$ 
& $\sigma_{p}$ & $2.0~\mathrm{m}$ \\

$\sigma_{v}$ & $0.4~\mathrm{m/s}$ 
& $v^{\rm d}_{\max}$ & $2.0~\mathrm{m/s}$ 
& $v^{\rm a}_{\max}$ & $2.5~\mathrm{m/s}$ \\

$u^{\rm d}_{\max}$ & $15.0~\mathrm{m/s^2}$ 
& $u^{\rm a}_{\max}$ & $20.0~\mathrm{m/s^2}$ 
& $r_{\rm d}^{\rm str}$ & $3.0~\mathrm{m}$ \\

$r_{\rm d}^{\rm com}$ & $5.0~\mathrm{m}$ 
& $r_{\rm d1}^{\rm safe}$ & $1.0~\mathrm{m}$ 
& $r_{\rm d2}^{\rm safe}$ & $0.8~\mathrm{m}$ \\

$k_{\rm ap}$ & $5$--$8$ 
& $k_{\rm ad}$ & $8$--$12$ 
& $k_{\rm aa}$ & $2$--$4$ \\

$\Delta t$ & $0.1~\mathrm{s}$ 
& $r_{\rm aa}^{\rm safe}$ & $1.0$--$2.0~\mathrm{m}$ 
& $r_{\rm aa}^{\rm avo}$ & $2.0$--$5.0~\mathrm{m}$ \\

$r_{\rm ad}^{\rm safe}$ & $0.9$--$1.8~\mathrm{m}$ 
& $r_{\rm ad}^{\rm avo}$ & $7.0$--$14.0~\mathrm{m}$ 
& $\alpha^{\pi}$ & $0.001$ \\

$\alpha^{q}$ & $0.001$ 
& $\tau$ & $0.005$ 
& $\gamma$ & $0.99$ \\

$k_{\rm sur}$ & $2.0$ 
& $k_{\rm dis}$ & $1.0$ 
& $k_{\rm pos}$ & $1.0$ \\

$k_{\rm ali}$ & $1.0$ 
& $r_{\rm att}^{\rm ref}$ & $3.5~\mathrm{m}$ 
& \multicolumn{2}{c}{} \\

\bottomrule 
\end{tabular} 
\end{table}

\subsection{Learning Process}\label{subsec:learning}
Policy learning is conducted in simulations involving 8 defenders and 3 attackers, with each episode terminated when all attackers are surrounded or when at least one attacker breaches the protected area. 
During the first 2000 episodes, the attacking strategy is deterministic. 
This setup allows the policy to first learn an effective herding behavior under a stable adversarial condition. 
Subsequently, diverse attacking strategies are gradually introduced to enhance policy generalization. 

To validate the superiority of the state-dependent reward function in Eq.~\eqref{reward_function}, we compare it with a static-weight counterpart using the same weights in two types of policy-learning simulations. 
These simulations differ in the random initialization range of the defenders' center, which is marked by light-green regions in the right panels of Fig.~\ref{Learning_pro}. 
In the first setting, referred to as front initialization, the defenders are initialized between the attacking swarm and the protected area, allowing them to establish a defensive barrier without substantial posture adjustment. 
As shown in Fig.~\ref{Learning_pro}(a), the formation gradually reshapes from a nearly linear configuration into a closed circular configuration as the attackers approach. 
Under this relatively favorable initialization, both reward functions effectively guide policy learning, and their success rates gradually converge to $100\%$ as the actor and critic networks are refined. 
These results validate the effectiveness of the four reward components shared by the two designs.

However, the performance consistency between the two reward functions diminishes when the defenders are initialized behind the protected area.
This setting is more challenging because the defensive formation must simultaneously adjust its position, orientation, and opening despite its limited maneuverability. 
As shown in Fig.~\ref{Learning_pro}(b), the policy learned with the static-weight reward function struggles to balance these objectives, resulting in a success rate below $60\%$. 
A typical failure mode is that the formation repeatedly intercepts the attackers without completing the encirclement. 
Specifically, the persistent effects of $f^{\rm ali}$ and $f^{\rm pos}$ drive the formation to remain aligned with the line connecting the attacking swarm and the protected area. 
Consequently, the formation rotates around the protected area instead of prioritizing the enclosure of the attackers.
In contrast to its counterpart, the state-dependent reward function remains effective under this more challenging initialization condition. 
{In} the early stage, the defensive formation approaches and aligns with the attacking swarm to provide a favorable spatial configuration for the subsequent surrounding task. 
As the attackers become partially constrained, the policy places a higher priority on the surrounding task. 
By adaptively activating different reward components according to the current herding situation, the state-dependent reward function enables the learned policy to achieve more efficient herding performance.

\begin{figure*}[t]
    \centering
    \includegraphics[width=0.95\textwidth]{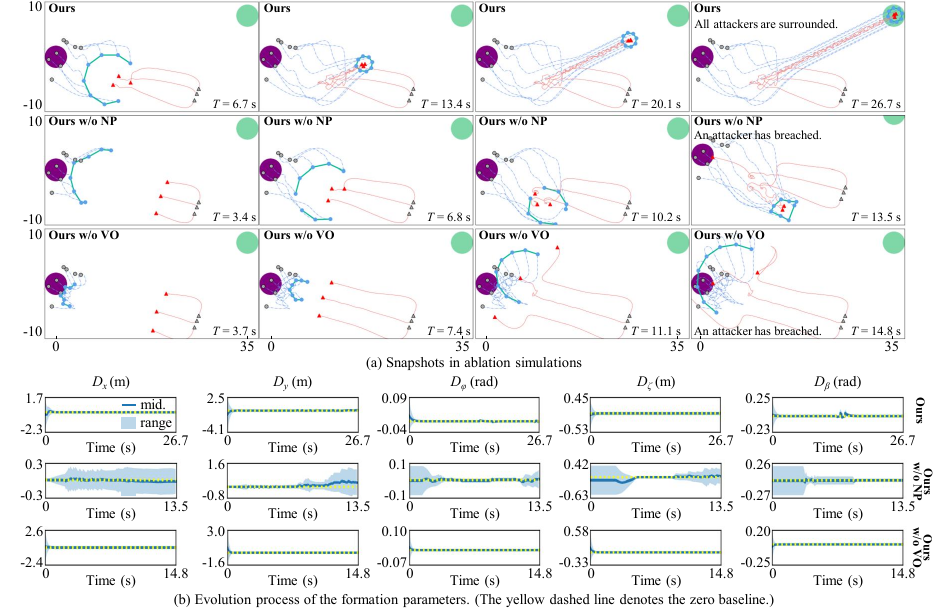}
    \caption{Ablation study results. 
    (a) Simulation results (top: ours, middle: ours \emph{w/o} NP, bottom: ours \emph{w/o} VO). 
    (b) Evolution of the formation parameters (top: ours, middle: ours \emph{w/o} NP, bottom: ours \emph{w/o} VO).}
    \label{ablation}
\end{figure*}

\subsection{Ablation Studies}\label{subsec:ablation}
This subsection evaluates the effectiveness of the negotiation protocol and velocity observations, and further examines the robustness of our method to communication failures and barrier-threshold violations.

We consider two ablated variants. 
In the first variant, the negotiation protocol is disabled (\textit{Ours w/o NP}), and each defender independently updates its local parameter estimate. 
To ensure that the defenders can initially organize into a coherent formation, their estimates of $\bm{p}_{\rm c}$ are initialized consistently, while the remaining parameter estimates deviate from each other. 
To evaluate the consensus error, we define a deviation set 
$D_{\sigma} = \big\{ \sigma_i - \bar{\sigma} \bigm| i \in \mathcal{V}_{\mathrm{d}} \big\}$, where $\bar{\sigma} = \frac{1}{N_{\mathrm{d}}} \sum_{j \in \mathcal{V}_{\mathrm{d}}} \sigma_{j}$ represents the mean parameter estimate across the defenders, and $\sigma_i$ specifies the parameter value estimated by defender $i$. 
In the second variant, the velocity observations are omitted (\textit{Ours w/o VO}). 
Accordingly, $\Vert\bm{v}_{\rm att}\Vert$, $\cos(\varphi_{\bm{v}_{\rm att}})$, and $\sin(\varphi_{\bm{v}_{\rm att}})$ are fixed to {$v^{\rm a}_{\rm max}+v^{\rm d}_{\rm max}$}, $1.0$, and $0.0$, respectively. 

The ablation results are shown in Fig.~\ref{ablation}, where identical initial robot positions are used across the three simulations. 
As shown in Fig.~\ref{ablation}(a), the attackers are successfully surrounded only when both the negotiation protocol and velocity observations are enabled. 
Without negotiation, the local parameter estimates rapidly diverge, causing severe formation deformation and gaps through which attackers can escape. 
Without velocity observations, the formation remains coherent but responds prematurely to attacker maneuvers and fails to adjust its parameters appropriately. 
These results confirm the necessity of both components.
Communication robustness is further evaluated through 100 randomized trials under different communication radii and packet loss rates, as reported in Table~\ref{tab:communicate_analysis}. 
When $r_{\rm d}^{\rm com}=4.0\,\mathrm{m}$, the success rate decreases from $73\%$ to $43\%$ as the packet loss rate increases from $20\%$ to $80\%$. 
By contrast, when $r_{\rm d}^{\rm com}=6.0\,\mathrm{m}$, the success rate remains at $61\%$ even under an $80\%$ packet loss rate. 
These results indicate that the negotiation protocol can tolerate moderate packet losses when sufficient local connectivity is maintained, whereas sparse communication may cause inconsistent local estimates and distorted formation segments.
\begin{table}[!t]
\centering
\setlength{\tabcolsep}{12pt}
\caption{Success Rates (\%) Under Different Communication Radii and Packet Loss Rates.}
\label{tab:communicate_analysis}
\begin{tabular}{ccccc}
\toprule
\multirow{2}{*}{Communication radius} 
& \multicolumn{4}{c}{Packet loss rate $\lambda_{\rm com}$} \\
\cmidrule(lr){2-5}
& 0.20 & 0.40 & 0.60 & 0.80 \\
\midrule
$r_{\rm d}^{\rm com}=4.0$ $\rm{m}$& 73 & 59 & 45 & 43 \\
$r_{\rm d}^{\rm com}=5.0$ $\rm{m}$& 86 & 68 & 67 & 56 \\
$r_{\rm d}^{\rm com}=6.0$ $\rm{m}$& 99 & 90 & 82 & 61 \\
\bottomrule
\end{tabular}
\end{table}

We also conduct a sensitivity analysis of the barrier-threshold assumption and the correction mechanism introduced in Appendix~\ref{correction}. 
Specifically, the effective barrier threshold is defined as $\hat{r}_{\rm d}^{\rm str}=\lambda_{\rm str}r_{\rm d}^{\rm str}$, where $\lambda_{\rm str}\in[0.5,1.0]$. 
Here, $\lambda_{\rm str}=1.0$ corresponds to the nominal barrier assumption used during policy learning. 
For each value of $\lambda_{\rm str}$, 100 trials are performed to evaluate the herding success rate, and the statistical results are summarized in Table~\ref{tab:barrier_sensitivity}. 
The correction-only case degrades significantly when $\lambda_{\rm str}\leq0.7$, because the reduced inter-defender spacing limits the spatial coverage of the defensive formation. 
Although the correction mechanism can mitigate moderate threshold violations, it cannot fully compensate for severe reductions in the effective barrier threshold. 
To mitigate this, we introduce additional defenders to restore the group coverage required for forming a valid defensive barrier. 
With two additional defenders, the success rate remains no lower than $85\%$ even at $\lambda_{\rm str}=0.5$. 
With four or six additional defenders, success rates of at least $97\%$ are maintained across all tested thresholds. 
Therefore, the barrier threshold violations can be effectively {compensated for} by the correction mechanism with additional defenders.

\begin{table}[!t]
\centering
\setlength{\tabcolsep}{6pt}
\caption{Success Rates (\%) Under Different Barrier Thresholds.}
\label{tab:barrier_sensitivity}
\begin{tabular}{lcccccc}
\toprule
\multirow{2}{*}{Strategy} 
& \multicolumn{6}{c}{Threshold coefficient $\lambda_{\rm str}$} \\
\cmidrule(lr){2-7}
& 1.00 & 0.90 & 0.80 & 0.70 & 0.60 & 0.50 \\
\midrule
No correction
& \textbf{100} & 72 & 64 & 57 & 55 & 44 \\
Correction
& \textbf{100} & \textbf{99} &  \textbf{92} & 55 & 22 & 3 \\
Correction $+2$ defenders 
& \textbf{100} & \textbf{100} & \textbf{100} &  \textbf{97} &  \textbf{87} &  \textbf{85} \\
Correction $+4$ defenders 
& \textbf{100} & \textbf{100} & \textbf{100} & \textbf{100} & \textbf{100} &  \textbf{97} \\
Correction $+6$ defenders 
& \textbf{100} & \textbf{100} & \textbf{100} & \textbf{100} & \textbf{100} & \textbf{99} \\
\bottomrule
\end{tabular}
\end{table}

\begin{figure*}[!t]
    \centering
    \includegraphics[width=0.99\textwidth]{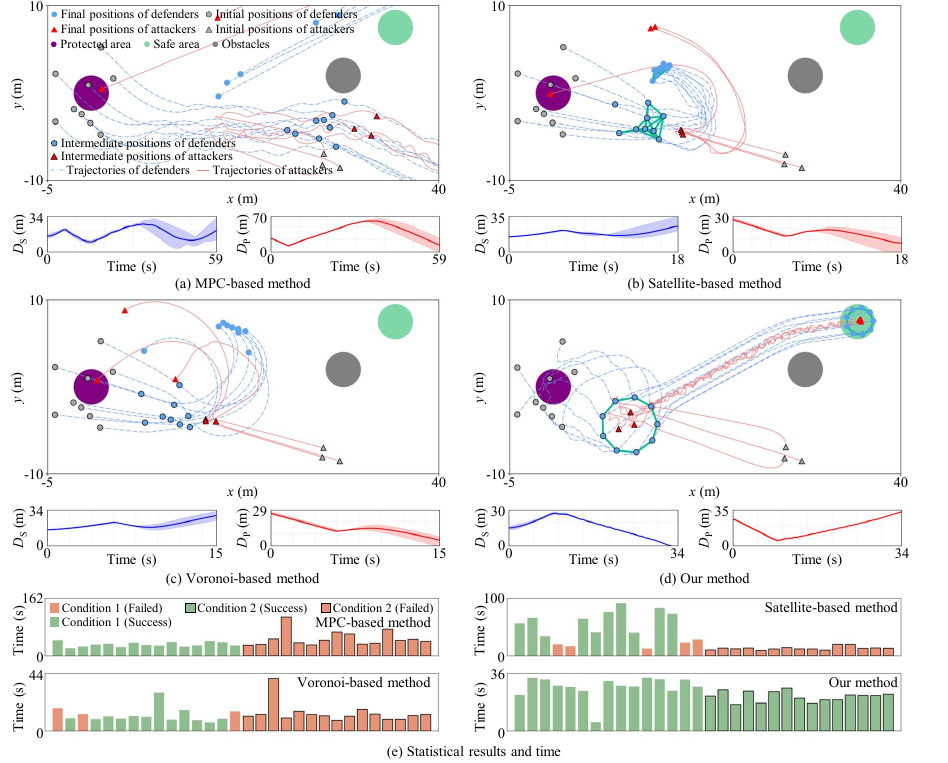}
    \caption{Comparative results against the baseline methods proposed in~\cite{deng2020multi, Chen2024RAL, chen2024multipleherding}.
    In (a)--(d), we depict results for the four evaluated methods, with each panel organized as follows: the top subfigure displays the trajectories, and the bottom subfigure presents the performance metrics. 
    In (e), we provide statistical results across 30 simulation trials evaluated under two distinct scenarios: Condition 1, where the attackers possess inferior maneuverability ({$u^{\rm a}_{\rm max} = 10.0$} {$\mathrm{m/s^2}$} and {$v^{\rm a}_{\rm max} = 1.5$} {$\mathrm{m/s}$}), and Condition 2, where the attackers possess superior maneuverability ({$u^{\rm a}_{\rm max} = 20.0$} {$\mathrm{m/s^2}$} and {$v^{\rm a}_{\rm max} = 2.5$} {$\mathrm{m/s}$}).}
    \label{comparation}
\end{figure*}

\subsection{Comparative Results}
In this subsection, we benchmark our method against the three baseline methods~\cite{deng2020multi, Chen2024RAL, chen2024multipleherding} over 30 independent trials. 
Simulation settings for this evaluation are presented as follows. 
First, Gaussian noise is incorporated into the observed information about the attackers only when evaluating our method, whereas the baseline methods are provided with ground-truth information. 
Second, given that the MPC-based method in~\cite{Chen2024RAL} requires the defenders to form a polygonal formation around each individual attacker, the numbers of defenders and attackers are set to 9 and 3, respectively. 
Finally, herding performance is evaluated by $D_{\rm S}$ and $D_{\rm P}$, denoting the attackers' distances to the safe and protected areas, respectively. 
Accordingly, $D_{\rm S} = 0$ denotes task success, while $D_{\rm P} = 0$ signifies task failure.
\begin{figure*}[t]
    \centering
    \includegraphics[width=0.96\textwidth]{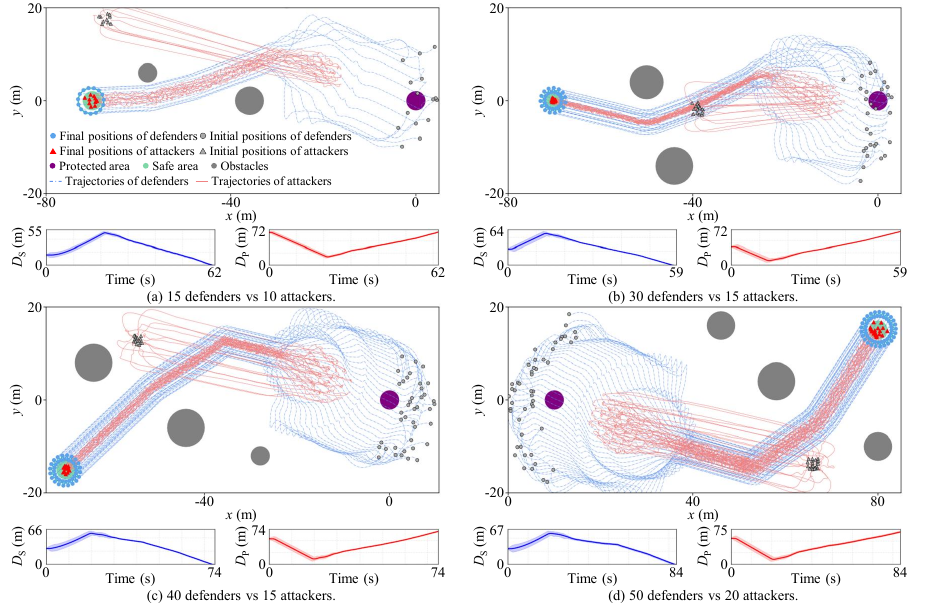}
    \caption{Generalization of the learned policy across different group sizes. 
    In each subfigure, the layout is as follows. Top: Simulation results. Bottom: Performance metrics.}
    \label{simulation3}
\end{figure*}

The three baselines require defenders to possess superior maneuverability. 
To verify their theoretical validity, we first evaluate these baselines across 15 trials under a relaxed scenario, where the attackers' maximum acceleration ({$u_{\mathrm{max}}^{\mathrm{a}}$}) and velocity ({$v_{\mathrm{max}}^{\mathrm{a}}$}) are restricted to 10.0 {$\mathrm{m/s^2}$} and 1.5 {$\mathrm{m/s}$}, respectively. 
As illustrated in Fig.~\ref{comparation}(e), all methods achieve high success rates under this condition. 
Compared with the other two baselines, the MPC-based baseline exhibits more robust performance and higher success rates. 
This comparative advantage demonstrates that explicitly maintaining organized formations can significantly enhance the collective capabilities of MRSs. 
Similar to our method, the MPC-based method drives defenders to form a polygonal formation around each attacker, resulting in multiple isolated subformations. 
This baseline method differs from ours in that it directly drives the defenders to positions around the attackers, rather than gradually encircling them from an open-arc formation. 
Although the attackers still execute aggressive evasive maneuvers as the defenders approach, the surrounding task is ultimately accomplished largely due to the defenders' superior maneuverability.

Subsequently, we evaluate all the methods under the more challenging scenario (Condition 2), where the attackers possess superior maneuverability. 
As illustrated in Fig.~\ref{comparation}(a)--(c), with the increase in the attackers' maneuverability, all the baselines fail to secure the protected area. 
This failure fundamentally stems from the fact that the baselines tend to perform surrounding actions too early. 
Specifically, the defenders prematurely approach the attackers before establishing a robust and continuous spatial encirclement. 
Consequently, the highly maneuverable attackers easily widen the gap, and then exploit the vast open space to bypass and leave the defenders behind. 
Ultimately, $D_{\rm P}$ gradually decreases to zero, indicating that the attackers have breached the protected area.

Conversely, our method achieves robust success by fundamentally transforming a speed-dependent pursuit into {geometry-dependent containment.}
Furthermore, our method achieves reliable herding even when the observed information regarding the attackers is noisy. 
This resilience is primarily attributed to a persistent concave barrier interposed between the attackers and the protected area. 
Rather than directing the defenders to rapidly minimize their distance to the attackers, which may provoke aggressive and unpredictable evasive maneuvers, our method uses the concave geometry to exploit the attackers' inherent intent to intrude into the protected area.
Driven by this intent, the attackers naturally navigate deeper into the arc-shaped formation. 
By systematically restricting their spatial freedom without triggering extreme evasive reactions, the defenders smoothly transform this concave barrier into a closed topological trap, fully compensating for their inherent maneuverability disadvantage.

In summary, the comparative results demonstrate the superiority of our method, which can effectively compensate for the limited maneuverability of defenders. 

\subsection{Scalability Analysis}
In this work, both the defensive formation and the attacking swarm are encoded using the parametric representation. 
Consequently, the learned policy can generalize to large-scale scenarios with varying numbers of defenders and attackers. 
To validate this advantage, we conduct a series of simulations involving dozens of defenders and attackers. 
Specifically, Figs.~\ref{simulation3}(a)--(d) illustrate the following simulations: 15 defenders vs. 10 attackers; 30 defenders vs. 15 attackers; 40 defenders vs. 15 attackers; and 50 defenders vs. 20 attackers.
The simulation results demonstrate that the learned policy successfully completes the task across varying group sizes. 
During the defending phase, the attacking swarm is ultimately confined within the closed-loop formation, while in the guiding phase, the formation maintains its shape and size as all the attackers are guided to the designated safe area. 
\begin{figure*}[t]
    \centering
    \includegraphics[width=0.98\textwidth]{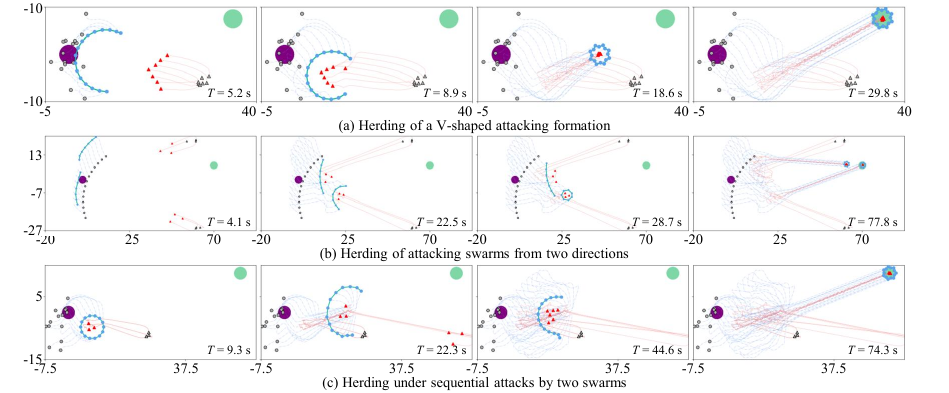}
    \caption{Robustness evaluation results. 
    The learned policy is further tested under three representative attack patterns.}
    \label{diverse_attack}
\end{figure*}

In addition to the four representative simulations shown above, we further conduct 100 randomized trials for each group configuration in Fig.~\ref{simulation3}. 
The success rates and runtime statistics are summarized in Table~\ref{tab:runtime}. 
The reported runtime is measured during online execution and excludes visualization and file I/O. 
For each group scale, the average runtime per step is calculated by dividing the average runtime per episode by the average number of execution steps. 
As shown in Table~\ref{tab:runtime}, the average runtime per step remains below $2.3\,\mathrm{ms}$ across all tested group scales. 
This computation time is far below the control period used in our real-robot implementation at $15\,{\mathrm{Hz}}$, i.e., approximately $66.7\,{\mathrm{ms}}$. 
The corresponding real-time ratio is therefore below $3.4\%$, demonstrating that our method satisfies real-time execution requirements. 
This efficiency mainly stems from the fixed-dimensional parametric representation, which enables the actor network to perform formation-level decision-making in a state space whose dimension does not increase with the number of robots.

\begin{table}[!t]
\centering
\footnotesize
\setlength{\tabcolsep}{6pt}
\renewcommand{\arraystretch}{1.12}
\caption{Runtime and Success Statistics Under Different Group Scales.}
\label{tab:runtime}
\begin{tabular}{ccccc}
\toprule
\makecell[c]{Group\\scale}
& \makecell[c]{Avg. steps\\per episode}
& \makecell[c]{Avg. runtime\\per episode ($\mathrm{s}$)}
& \makecell[c]{Avg. runtime\\per step ($\mathrm{ms}$)}
& \makecell[c]{Success\\rate (\%)} \\
\midrule
15D+10A & 515 & 0.760 & 1.475 & 99 \\
30D+15A & 394 & 0.892 & 2.264 & 97 \\
40D+15A & 375 & 0.447 & 1.191 & 99 \\
50D+20A & 463 & 0.501 & 1.082 & 97 \\
\bottomrule
\end{tabular}
\end{table}
Meanwhile, our method achieves success rates of no less than $97\%$ across all tested group scales and exhibits strong scalability. 
In the remaining failure cases, the outermost attacker may stay slightly outside the closed formation, although the formation geometry has already indicated a successful encirclement. 
This suggests that large-scale deployment requires a small safety margin between the commanded formation scale and the actual robot-level execution. 
In our implementation, we multiply $\zeta$ by a factor of $1.1$ to mitigate this issue. 
With this adjustment, the defensive formation provides wider effective coverage, and the corresponding outcome can change from failure to successful herding. 

\subsection{Robustness Evaluation under Different Attacking Patterns}
To further evaluate the robustness of the learned policy, additional simulations are conducted under three organized attacking patterns, as shown in Fig.~\ref{diverse_attack}: a V-shaped formation, two sub-swarms approaching simultaneously from different directions, and two sub-swarms appearing sequentially.

In the first case, the V-shaped formation enlarges the spatial distribution of the attackers and thereby increases the difficulty of encirclement. 
As shown in Fig.~\ref{diverse_attack}(a), the defensive formation correspondingly expands, progressively surrounds the attackers, and guides them toward the safe area.

In the second case, two attacking swarms approach the protected area from different directions, requiring the defenders to handle multiple spatially separated threats simultaneously. 
As shown in Fig.~\ref{diverse_attack}(b), we first divide the defenders into two subgroups according to the spatial distribution of the attackers. 
Each subgroup independently adjusts its configuration based on the policy to herd its assigned attacking swarm. 
The results show that the policy remains applicable at the subgroup level and enables both attacking swarms to be successfully herded.

In the third case, two attacking swarms appear sequentially, which introduces temporal uncertainty into the herding process. 
After the first swarm is surrounded, the appearance of the second swarm changes the overall threat distribution. 
The formation opening is therefore adjusted to $\pi/3$, allowing the defenders to guide the first swarm toward the newly appearing swarm while preventing it from escaping. 
Once the two swarms merge, they are treated as a single integrated swarm, and the learned policy is applied to the merged group. 
As shown in Fig.~\ref{diverse_attack}(c), the defensive formation closes again and surrounds all attackers. 
This result demonstrates that the learned policy can respond not only to spatially different attacking patterns but also to sequentially appearing threats.

These results indicate that our method is not limited to the nominal attacking pattern used during policy learning. 
Instead, the learned policy captures the geometric relationship between the two groups by representing them in a parametric manner. 
Therefore, even when the attackers exhibit different behaviors, the policy remains effective. 

\begin{figure*}[t]
    \centering
    \includegraphics[width=1.0\textwidth]{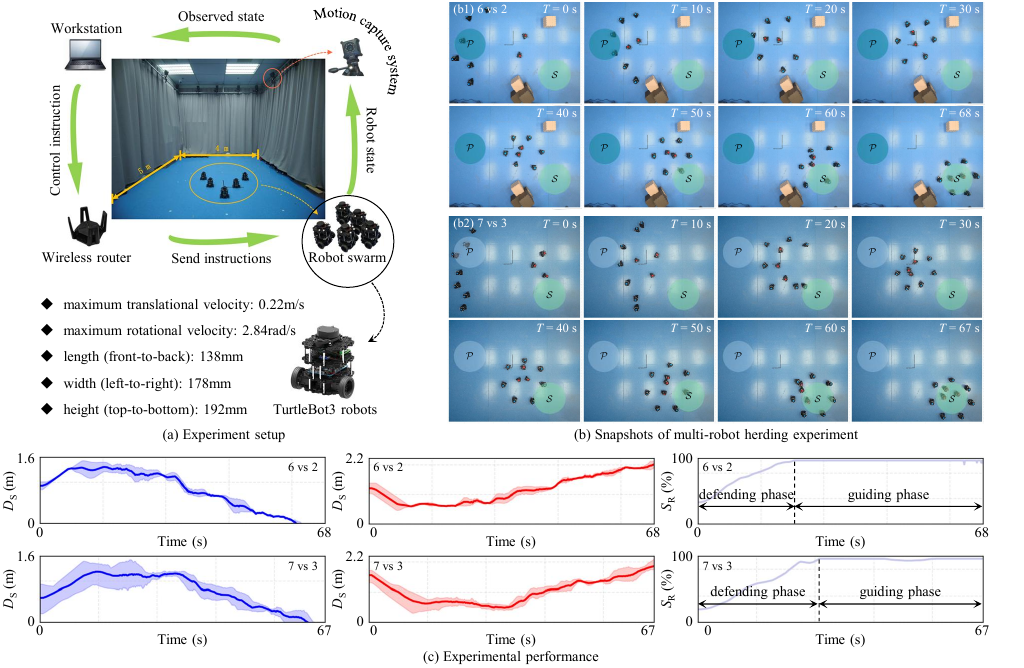}
   \caption{Experimental setup and results of the multi-robot herding validation.}
     \label{exp}
\end{figure*}

\section{Experiment Results}\label{sec:experiments}
To further demonstrate the practical effectiveness of our method, we conduct experiments on an indoor platform consisting of 10 TurtleBot3 robots (7 defenders and 3 attackers). 

\subsection{Experiment Setting}
As shown in Fig.~\ref{exp}(a), the experiment platform measures 4 {$\mathrm{m}$} $\times$ 5 {$\mathrm{m}$} and is equipped with a motion capture system to acquire the poses of the robots in real time. 
Since the TurtleBot3 robot is subject to nonholonomic kinematic constraints, the control law $\bm{u}_i$, designed for an omnidirectional robot model, is converted into the linear and angular velocity commands $v_i$ and $\omega_i$, respectively, following the method in~\cite{Zhao2018General}.
Accordingly, the control commands are computed as
$v_i=[\cos\vartheta_i,\sin\vartheta_i]\bm{u}_i$
and
$\omega_i=[-\sin\vartheta_i,\cos\vartheta_i]\bm{u}_i$,
where $\vartheta_i$ denotes the heading angle of the $i$th robot.
In addition, the maximum linear and angular velocities of the defenders are set to ${v^{\rm d}_{\max}} = 0.15$ $\mathrm{m/s}$ and ${\omega^{\rm d}_{\max}} = 0.8$ $\mathrm{rad/s}$, respectively, while those of {attackers} are set to ${v^{\rm a}_{\max}} = 0.2$ $\mathrm{m/s}$ and ${\omega^{\rm a}_{\max}} = 1.5$ $\mathrm{rad/s}$. 
The remaining parameters are kept consistent with those used in the simulation. 
These commands are then transmitted to the robots via the ROS~2 network at 15 {$\mathrm{Hz}$}. 
Our tests show that the computation time for these commands under distributed conditions is no more than 3 {$\mathrm{ms}$}, which is far below the approximately $66.7\,\mathrm{ms}$ control period.

To facilitate visual identification, {the defender robots} are marked with blue indicators, while {the attacker robots} are marked with red indicators.
In addition to $D_{\rm S}$ and $D_{\rm P}$, we define the surrounding ratio $S_{\rm R}$ to better evaluate the herding performance in terms of the defensive formation's enclosure relative to the attackers.

\subsection{Multi-Robot Herding Experiment}
As shown in Fig.~\ref{exp}(b), two experimental scenarios are designed to validate our method. 
In these two scenarios, the defenders are initially positioned on the left, while the attackers are located on the right. 
The cyan region on the left represents the protected area, and the green region on the right denotes the safe area.

The first scenario demonstrates a herding task where 6 defenders coordinate to herd 2 attackers within an environment containing obstacles. 
Throughout the experiment, the attackers persistently attempt to breach the protected area. 
Despite being initialized behind the protected area, the defenders rapidly establish a defensive line midway between the attackers and the protected area during the initial interval of 0--10 {$\mathrm{s}$}.
As the attackers advance toward the protected area, the defenders coordinate to push the defensive formation forward to intercept them.
Once the surrounding advantage is sufficiently enlarged, the two defenders at the ends move toward each other to form a closed-loop formation at 23 {$\mathrm{s}$} to surround the attackers. 
As a result, the defensive formation successfully transports them to the safe area by 68 {$\mathrm{s}$}. 
The performance metrics in the first experiment are illustrated in the upper panel of Fig.~\ref{exp}(c). 
As observed, although the attackers initially approach the protected area, they are effectively herded to the safe area and kept at a safe distance from the protected area. 
Furthermore, when approaching obstacles, the formation proactively executes a collective avoidance maneuver to preserve its geometric rigidity. This strategy ensures that the surrounding ratio remains consistently at 100 \%, thereby guaranteeing robust task execution.

In comparison to the first experiment, we scale up the group sizes in the second to validate the generalization of our method. 
Specifically, this experiment involves 7 defenders coordinating to herd 3 attackers. 
As presented in Fig.~\ref{exp}(b), the expansion of the group sizes inevitably prolongs the convergence time of the defensive formation. 
However, due to the establishment of a preliminary defensive barrier, the attackers are consistently prevented from breaching the {defensive line.}  
This provides a temporal window for the defenders at the two ends to execute outflanking maneuvers. 
Consequently, a closed-loop formation is achieved around 30 {$\mathrm{s}$}, indicating the transition from the defending phase to the guiding phase. 
Finally, the attackers are successfully herded to the safe area at 67 {$\mathrm{s}$}.

These experiments demonstrate the strong potential of our method for practical herding applications. 
The video of these experiments can be found in the supplementary material. 
In addition to the two experiments discussed above, four extra experimental trials are provided in the video to further validate the performance and effectiveness of our method.

\section{Conclusion}\label{sec:conclusion}
In this paper, we proposed a task-oriented formation decision method for MRSs, with a focus on the herding task. 
First, a parametric representation was introduced to encode the formation shape, which transformed the formation decision task into a parameter optimization problem. 
Second, we developed an RL-based policy to adjust these parameters based on task requirements. 
The learned policy determined an appropriate formation shape based on the available observations. 
Finally, we compared our method with three baselines~\cite{deng2020multi, Chen2024RAL, chen2024multipleherding} through 30 simulation trials and validated it in physical experiments involving 10 TurtleBot3 robots. 
The results show that our method can effectively compensate for the defenders' limited maneuverability, even when the attackers' strategies are unknown. 
Nevertheless, the proposed method still has several limitations. 
Although the negotiation protocol can tolerate moderate packet losses with sufficient local connectivity, sparse communication may still lead to inconsistent parameter estimates and local formation distortion. 
In future work, we will develop a fully distributed formation-decision and herding method to reduce dependence on centralized sensing.

\appendices
\section{Control Law for Individual Robots} \label{Control law}
Given the determined vector $\bm{\theta}$, the defenders are required to {be distributed} evenly within the formation shape. 
To prevent potential inter-robot collisions when the commanded inter-defender spacing satisfies $\zeta < r_{\rm d1}^{\rm safe}$, we propose a safety-aware distribution mapping $\mathcal{F}: \bm{\theta} \rightarrow \mathcal{P}^{*} = \{\bm{p}_i^{*}\}_{i=1}^{N_{\rm{d}}}$. 
This mapping yields a two-layer distribution. 
For the two-layer distribution, the adjacent defenders in the innermost layer should satisfy the collision-avoidance requirement. 
Therefore, when $\zeta<r_{\rm d1}^{\rm safe}/2$, the value of $\zeta$ is set to $r_{\rm d1}^{\rm safe}/2$. 
Accordingly, the distribution mapping is defined as:
\begin{equation*}
    \hat{\bm{q}}_i = 
    \begin{cases}
    \bm{q}_i +(-1)^i\Delta
    \begin{bmatrix} 
    \cos\left(\psi_i + \frac{\pi}{2}\right) \\ 
    \sin\left(\psi_i + \frac{\pi}{2}\right) 
    \end{bmatrix}, 
    & \zeta < r_{\rm d1}^{\rm safe}, \\[2mm]
    \bm{q}_i, 
    & \zeta \geq r_{\rm d1}^{\rm safe}.
    \end{cases}
\end{equation*}
Here, $\Delta = \frac{1}{2}\sqrt{(r_{\rm d1}^{\rm safe})^2 - \zeta^2}$ represents the normal offset required to maintain the safety distance $r_{\rm d1}^{\rm safe}$. The base relative positions $\bm{q}_i$ are recursively generated by:
\begin{equation*}
    \bm{q}_i = 
    \begin{cases}
         \bm{0}, & i=1  \\
        \bm{q}_{i-1} - \zeta \begin{bmatrix} \cos{\psi_i} \\ \sin{\psi_i} \end{bmatrix}, & i=2,3,\dots,N_{\rm{d}}
    \end{cases}
\end{equation*}
where $\psi_i = \varphi + \frac{\pi}{2} + \dfrac{(2i-N_{\rm{d}}-2)\beta}{2(N_{\rm{d}}-1)}$.

Building on our previous work in~\cite{chen2023collision}, once defender $i$ estimates its desired position, it computes its control law $\bm{u}_i$:
\begin{equation*}
    \bm{u}_i = \bm{u}_i^{\rm f}  + \bm{u}_i^{\rm a}+ \bm{u}_i^{\rm c}
    \label{Equ_control_law}
\end{equation*}
where $\bm{u}_i^{\rm f}$, $\bm{u}_i^{\rm a}$, and $\bm{u}_i^{\rm c}$ represent terms contributed by formation maintenance, collision avoidance, and connectivity preservation, respectively. 

The first term $\bm{u}_i^{\rm f}$ is detailed as follows:
\begin{equation*}
    \bm{u}_i^{\rm f} = k_{\rm p}(\bm{p}_i^{*}-\bm{p}_i) + k_{\rm v}(\dot{\bm{p}}_{\rm c}-\bm{v}_i)
\end{equation*}
where constants $k_{\rm p}$ and $k_{\rm v}$ denote the weights for eliminating the position and velocity {errors}, respectively. 
Each defender can derive the velocity of the formation center, ${\dot{\bm{p}}_{\rm c}}$, based on its estimate of the formation parameters, $\hat{\bm{\theta}}_i$. 

The second term $\bm{u}_i^{\rm a}$ is detailed as follows: 
\begin{equation*}
\begin{aligned}
    \bm{u}_i^{\rm a}
    =
    -k_{\rm a}
    \sum_{j\in \mathcal{O}\cup \mathcal{N}_i}
    \frac{\partial g\left(r_{ij},r_{\rm d1}^{\rm safe},r_{\rm d2}^{\rm safe}\right)}
    {\partial \bm{p}_i}
\end{aligned}
\end{equation*}
where constant $k_{\rm a}$ denotes the weight for obstacle and inter-agent collision avoidance, and $\mathcal{N}_i = \{j \in \mathcal{V}_{\rm d} : (i,j) \in \mathcal{E}\}$ indicates the neighbor set. 
Notably, this term allows the formation to temporarily deviate from its desired configuration to avoid collisions.
The first term ensures rapid convergence back to the desired configuration after the avoidance maneuvers. 

The third term $\bm{u}_i^{\rm c}$ is responsible for connectivity preservation in distributed systems. 
To achieve this, the defender network is represented by an undirected graph $\mathcal{G} = (\mathcal{V}_{\rm{d}},\mathcal{E})$, where the edge set $\mathcal{E} = \{(i,j) : i,j \in \mathcal{V}_{\rm{d}},\Vert\bm{p}_i - \bm{p}_j\Vert \leq r_{\rm d}^{\rm com} \}$ represents communication links, and $r_{\rm d}^{\rm com}$ denotes the communication range of defenders. 
The adjacency matrix of $\mathcal{G}$ is defined as $\bm{A}(\mathcal{G}) = \{a_{ij}\}_{N_{\rm{d}}\times N_{\rm{d}}}$, where $a_{ij}=h(\Vert\bm{p}_i - \bm{p}_j\Vert, r_{\rm d}^{\rm com})$ is calculated using the following function:
\begin{equation*}
    h(r,r_0) = \begin{cases}
        \dfrac{1}{2} \left[1+\cos\left(\pi\dfrac{r}{r_0}\right)\right], & r\leq r_0 \\
        0, & r>r_0
    \end{cases}
\end{equation*}
when $\Vert\bm{p}_i-\bm{p}_j\Vert$ approaches $r_{\rm d}^{\rm com}$, $a_{ij}$ tends to 0. 
This design reflects the continuous attenuation characteristics of signal strength with distance in practice.
The Laplacian matrix {is defined as} $\bm{L}(\mathcal{G}) = \bm{D}(\mathcal{G}) - \bm{A}(\mathcal{G})$, where $\bm{D}(\mathcal{G}) = {\rm diag}\{\sum_{j=1}^{N_{\rm{d}}} a_{ij}\}$.
The graph $\mathcal{G}$ is connected if the second-smallest eigenvalue $\lambda_2(\mathcal{G})>0$~\cite{Fiedler1973}. Otherwise, the graph is disconnected. 
As a result, the term $\bm{u}_i^{\rm c}$ is detailed as follows: 
\begin{equation*}
    \bm{u}_i^{\rm c} = - k_{\rm c}\sum_{\mathcal{C}\in\mathcal{Q}_i} \dfrac{\partial g\left(\lambda_2(\mathcal{G}_{\mathcal{C}}),\delta_1,\delta_2\right)}{\partial \lambda_2(\mathcal{G}_{\mathcal{C}})} \dfrac{\partial \lambda_2(\mathcal{G}_{\mathcal{C}})}{\partial \bm{p}_i}
\end{equation*}
where constant $k_{\rm c}$ denotes the weight for connectivity preservation, and $\mathcal{Q}_i$ denotes the maximal cliques set associated with defender $i$. 
Physically, each clique $\mathcal{C} \in \mathcal{Q}_i$ represents a local, fully connected communication sub-swarm where all involved defenders can mutually interact. 
The fundamental principle of $\bm{u}_i^{\rm c}$ is to guarantee global network integrity by rigidly maintaining these physical sub-swarms. 
For a detailed derivation and theoretical analysis of $\bm{u}_i^{\rm c}$, the reader is referred to our previous work~\cite{chen2023collision}.

\section{Theoretical Analysis}
\subsection{Proof of Theorem 1} \label{proof1}
Define the Lyapunov function
\[
    V = \frac{1}{2}\sum_{i\in\mathcal{V}_{\rm d}} 
    \Vert \hat{\bm{\theta}}_i-\bm{\theta}^* \Vert^2.
\]
Its derivative is
\[
    \dot{V}
    = \sum_{i\in\mathcal{V}_{\rm d}}
    \big(\hat{\bm{\theta}}_i-\bm{\theta}^{*})^\top 
    \dot{\hat{\bm{\theta}}}_i.
\]
Substituting~\eqref{Equ_negotiation},
\[
\begin{aligned}
    \dot{V}
    =& \sum_{i}
    (\hat{\bm{\theta}}_i-\bm{\theta}^{*})^\top\pi_i(\bm{s}_i)
    + c_{\theta}(\bm{\theta}^*)^\top 
    \sum_{i}\sum_{j\in\mathcal{N}_i}(\hat{\bm{\theta}}_j - \hat{\bm{\theta}}_i) \\
    &+ c_{\theta}\sum_{i}\sum_{j\in\mathcal{N}_i}
    \hat{\bm{\theta}}_i^\top(\hat{\bm{\theta}}_j - \hat{\bm{\theta}}_i).
\end{aligned}
\]

The first term is non-positive. 
Since the graph is undirected,
\[
    \sum_{i}\sum_{j\in\mathcal{N}_i}
    (\hat{\bm{\theta}}_j - \hat{\bm{\theta}}_i)=\bm{0},
\]
so the second term is zero.
For the third term,
\[
\begin{aligned}
    &c_{\theta}\sum_{(i,j)\in\mathcal{E},i<j}
    \left[
        \hat{\bm{\theta}}_i^\top(\hat{\bm{\theta}}_j - \hat{\bm{\theta}}_i)
        + \hat{\bm{\theta}}_j^\top(\hat{\bm{\theta}}_i - \hat{\bm{\theta}}_j)
    \right]\\
    &=-c_{\theta}
    \sum_{(i,j)\in\mathcal{E},i<j}
    \Vert \hat{\bm{\theta}}_i - \hat{\bm{\theta}}_j \Vert^2
    \le 0.
\end{aligned}
\]

Thus $\dot{V}\le 0$, and $\dot{V}=0$ if $\hat{\bm{\theta}}_i=\bm{\theta}^*$ for all defenders.  
Therefore, $\lim\limits_{t\rightarrow\infty} \hat{\bm{\theta}}_i = \bm{\theta}^{*}, \forall i \in \mathcal{V}_{\rm d}$. 

\subsection{Complexity Analysis} \label{Complexity}
Here, we analyze the computational complexity of our method to verify its real-time feasibility. 
Crucially, the computational load per robot {does not directly scale with} the total number of robots ($N_{\rm d}$ and $N_{\rm a}$).
This exceptional scalability fundamentally stems from two core designs: a fixed-dimensional state representation and a {distributed} communication topology. 
The detailed analysis of the two main computational components, actor network inference and the negotiation protocol, is as follows.

During the online phase, each defender locally generates {an} action using the deployed actor network. 
Specifically, the computational complexity of forward inference is $O(\sum_{l=1}^{L-1} n_l n_{l+1})$, where $L$ is the total number of layers and $n_l$ is the number of neurons in layer $l$. 
For this actor network, the complexity is $O(12 \times 128 + 128 \times 128 + 128 \times 5)$.
This calculation scales independently of the group sizes and simplifies to a small constant time complexity of $O(1)$.

Following the inference process, the defenders use a negotiation protocol to reach a consensus on their estimated formation parameters. 
The protocol requires each defender $i$ to communicate with its neighbors $j \in \mathcal{N}_i$ to compute the consensus term $c_{\theta}\sum_{j\in\mathcal{N}_i}(\hat{\bm \theta}_j - \hat{\bm \theta}_i)$. 
The complexity of this step for defender $i$ is proportional to its number of neighbors, yielding $O(|\mathcal{N}_i|)$. 
Since our system operates in a distributed manner, each defender interacts solely with its immediate neighbors. 
Therefore, the neighbor count $|\mathcal{N}_i|$ remains a bounded small integer that does not grow with the group size $N_{\rm d}$.

In summary, the forward inference of the neural network and the local negotiation protocol impose time complexities of $O(1)$ and $O(|\mathcal{N}_i|)$, respectively. 
Therefore, the overall online computational load remains strictly bounded and remarkably low, which readily satisfies real-time control requirements.

\section{DDPG Framework}\label{DDPG_framework}
The DDPG framework consists of three main components:
\begin{itemize}
\item[$\bullet$] Actor network $\bm{\pi}$.
The actor network is a fully connected neural network consisting of four layers with 12, 128, 128, and 5 neurons, respectively, where the input and output dimensions correspond to the state and action spaces. 
ReLU activation functions are used in the hidden layers, and a Tanh activation is applied to the output layer, followed by action scaling to satisfy the environment-specific action bounds. 
\item[$\bullet$] Critic network $\bm{q}$. 
The critic network is also fully connected and consists of four layers with {17, 128, 128, and 1 neurons, respectively.}
Its input is the concatenated state-action vector, and its output is a scalar Q-value. 
\item[$\bullet$] Target networks $\bm{\pi}_{\rm {tar}}$ and $\bm{q}_{\rm {tar}}$.
The target networks share the same architectures as the actor and critic networks, respectively, and are updated using soft updates to stabilize training. 
\end{itemize}

The loss functions are defined as: 
\begin{equation*}
\begin{cases}
J^{\rm q} = \sum_{\mathcal{B}_k \in \mathcal{B}} 
\big(\bm{q}(\bm{s}_k,\bm{a}_k)-\gamma \bm{q}_{\rm {tar}}(\bm{s}_{k+1}, \bm{\pi}_{\rm {tar}}(\bm{s}_{k+1}))-r \big)^2\\
J^{\rm \pi} = -\sum_{\mathcal{B}_k\in\mathcal{B}} \bm{q}\big(\bm{s}_{k},\bm{\pi}(\bm{s}_{k})\big)\\
\end{cases}
\end{equation*}
where $\gamma$ denotes the discount factor. 
In addition, the network parameters are updated according to: 
\begin{equation*}
\begin{cases}
\bm q^{t+1} = \bm q^t - \alpha^{\rm q}\nabla_{\bm q} J^{\rm q},\\
\bm \pi^{t+1} = \bm \pi^t - \alpha^{\rm \pi}\nabla_{\bm \pi} J^{\rm \pi},\\
\bm q^{t+1}_{\rm tar} = \tau \bm q^{t+1} + (1-\tau) \bm q^{t}_{\rm tar},\\
\bm \pi^{t+1}_{\rm tar} = \tau \bm \pi^{t+1} + (1-\tau) \bm \pi^{t}_{\rm tar},
\end{cases}
\end{equation*}
where $\alpha^{\rm q}$ and $\alpha^{\rm \pi}$ are the learning rates of the critic and actor networks, respectively. $\tau\ll 1$ is the soft update coefficient.

\section{Scale-Adaptive Policy Transfer Mechanism}\label{policy_transfer}
The objective is to enable a policy trained with a fixed number of robots to be transferable to scenarios with varying group sizes. 
Since both adversarial groups are represented as formation shapes with specific geometric characteristics, this transfer can be achieved through shape scaling. 
Accordingly, we introduce a scale-adaptive policy transfer mechanism based on shape-scaling. 
The online deployment includes two alternating steps: policy inference and physical execution. 
First, for policy inference, this mechanism maps scenarios with different group sizes into an equivalent representation at the learning scale, which is then fed into the actor network. 
Second, the inferred actions are scaled back to the deployment scale for physical execution. 
Specifically, the scale-related parameters include $\zeta$ and $r_{\rm att}$, for which two parameter groups are defined: $(\zeta^{\pi}, r_{\rm att}^{\pi})$ for policy inference and $(\zeta^{\rm dep}, r_{\rm att}^{\rm dep})$ for physical execution. 
Only two scaling operations are involved: $r_{\rm att}^{\rm dep} \rightarrow r_{\rm att}^{\pi}$ before policy inference and $\zeta^{\pi} \rightarrow \zeta^{\rm dep}$ during physical execution. 
The specific steps are presented as follows.

\subsubsection{Scaling for Policy Inference}
For an attacking swarm, the key geometric variable is its radius $r_{\rm att}$. 
From the perspective of formation geometry, the formation characteristic most closely related to $r_{\rm att}$ is the defensive arc length $L_{\rm d} = N_{\rm d}\zeta$. 
To match the scale of $r_{\rm att}$ with that of $\zeta^{\pi}$, the radius fed into the actor network is scaled as 
\begin{equation*}
r_{\rm att}^{\pi}
=
\frac{N_{\rm d}^{\pi}}{N_{\rm d}} r_{\rm att}^{\rm dep}
\end{equation*}
where $N_{\rm d}^{\pi}$ and $N_{\rm d}$ denote the numbers of defenders in learning and deployment, respectively. 

\subsubsection{Scaling for Physical Execution}
After the nominal parameter $\zeta^{\pi}$ is updated, the $N_{\rm d}$ defenders are required to physically realize an arc-shaped defensive barrier with a nominal length of $N_{\rm d}^{\pi}\zeta^{\pi}$.
Therefore, the scale parameter used to calculate the desired positions of defenders in Appendix A is given by
\begin{equation*}
\zeta^{\rm dep}
=\frac{N_{\rm d}^{\pi}}{N_{\rm d}}\zeta^{\pi}.
\end{equation*}
This operation enables the pre-trained policy to effectively generalize to scenarios with similar group sizes. 

However, in large-scale group scenarios, the actual attacking swarm can be much larger than the reference one used in the policy learning environment. 
In this case, the nominal defensive arc length determined at the learning scale may provide insufficient physical coverage.
To address this limitation, we introduce an additional factor to restore the physical coverage required in large-scale group scenarios. 
As a result, the scale parameter in large-scale group scenarios is formulated as
\begin{equation*}
\zeta^{\rm dep}
=\chi(r_{\rm att}^{\rm dep})\frac{N_{\rm d}^{\pi}}{N_{\rm d}}\zeta^{\pi}
\end{equation*}
where $\chi(r_{\rm att}^{\rm dep})
=
\max\big(\frac{r_{\rm att}^{\rm dep}}{r_{\rm att}^{\rm ref}},1\big)$ is the additional factor, and $r_{\rm att}^{\rm ref}$ denotes the reference radius of the attacking swarm in policy learning. 
If $r_{\rm att}^{\rm dep}>r_{\rm att}^{\rm ref}$, then $\chi(r_{\rm att}^{\rm dep})>1$, and the defensive arc is enlarged proportionally to recover sufficient physical coverage. 
Otherwise, $\chi(r_{\rm att}^{\rm dep})=1$, so the learning-scale defensive formation is sufficient for coverage and no further enlargement is required. 

\section{Correction Mechanism for Barrier-Threshold Violation}\label{correction}
In Assumption~\ref{Ass_1}, we assume that adjacent defenders create an impassable barrier for the attackers if the distance between them is less than a given threshold $r_{\rm d}^{\rm str}$. 
However, during online deployment, the effective barrier threshold may differ from its assumed value. 
Here, we focus on the more conservative case where the actual barrier threshold $\hat{r}_{\rm d}^{\rm str}$ is smaller than the assumed one, which means that $\hat{r}_{\rm d}^{\rm str}=\lambda_{\rm str} r_{\rm d}^{\rm str}$ and $0<\lambda_{\rm str}<1$. 
To restore the validity of the barrier assumption when $\zeta^{\rm dep}>\hat{r}_{\rm d}^{\rm str}$, additional defenders are required to provide a stronger group advantage and compensate for the reduced effective threshold. 

To implement this idea, a group-advantage factor is introduced as $\rho_{\rm g}=\max\big(\frac{N_{\rm d}}{N_{\rm d}^{\pi}},1\big)$.
When $\zeta^{\rm dep}>\hat{r}_{\rm d}^{\rm str}$, the scale parameter is corrected according to
\begin{equation*}
    \zeta^{\rm cor}
    =
    \min\Big(\frac{\zeta^{\rm dep}}{\rho_{\rm g}\lambda_{\rm str}},
    \hat{r}_{\rm d}^{\rm str}\Big).
\end{equation*}
As a result, the corrected spacing always satisfies $\zeta^{\rm cor}\leq \hat{r}_{\rm d}^{\rm str}$, which ensures that the barrier assumption remains valid under the reduced effective threshold. 
However, the correction mechanism cannot fully compensate for severe violations of the barrier-threshold assumption. 
When $\lambda_{\rm str}$ is too small, the effective barrier threshold is overly restrictive, and the defensive formation no longer has sufficient spatial coverage to establish an impassable barrier. 
In this case, additional defenders are required to increase the available group advantage and restore sufficient defensive coverage. 
The performance of the correction mechanism is presented in Section~\ref{subsec:ablation}. 

\bibliographystyle{IEEEtran}
\bibliography{reference}

@article{Gedefaw2025UAVReview,
  author  = {Gedefaw, Elisabeth Andarge and Abera, Nardos Belay and Abdissa, Chala Merga},
  title   = {A Review of Modeling and Control Techniques for Unmanned Aerial Vehicles},
  journal = {Eng. Rep.},
  volume  = {7},
  number  = {6},
  year    = {2025},
  note    = {{Art.} no. e70215}
}

@article{Guo2024JFR,
  author  = {Guo, Junlong and Li, Yakuan and Huang, Bo and Ding, Liang and Gao, Haibo and Zhong, Ming},
  title   = {An online optimization escape entrapment strategy for planetary rovers based on Bayesian optimization},
  journal = {J. Field Robot.},
  volume  = {41},
  number  = {8},
  pages   = {2518-2529},
  month   = {Dec.},
  year    = {2024}
}

@article{Ma2023Automatica,
  title   = {Intentional delay can benefit consensus of second-order multi-agent systems},
  journal = {Automatica},
  volume  = {147},
  year    = {2023},
  month   = {Jan.},
  issn    = {0005-1098},
  author  = {Qian Ma and Shengyuan Xu},
  note    = {{Art}. no. 110750}
}

@article{Fu2025AST,
  title   = {Adaptive safety attitude control of a hybrid {VTOL} {UAV} under transition flight subject to multiple faults and uncertainties},
  journal = {Aerosp. Sci. Technol.},
  volume  = {163},
  year    = {2025},
  month   = {Aug.},
  author  = {Yifang Fu and Ban Wang and Huimin Zhao and Mengqi Zhou and Ni Li and Zhenghong Gao},
  note    = {{Art.} no. 110284}
}

@article{Chen2025IoTJ,
  author  = {Chen, Jin and Li, Ming and Marcantoni, Matteo and Jayawardhana, Bayu and Wang, Yafei},
  journal = {IEEE Internet Things J.},
  title   = {Range-Only Distributed Safety-Critical Formation Control Based on Contracting Bearing Estimators and Control Barrier Functions},
  year    = {2025},
  volume  = {12},
  number  = {19},
  month   = {Oct.},
  pages   = {40968-40979}
}

@article{Liu2026ESWA,
  author  = {Qi Liu and Zhuoyang Song and Yuxin Liang and Zejian Xie and Songxin Zhang and Jiaxing Zhang and Yanjie Li},
  title   = {{CoRLHF}: Reinforcement learning from human feedback with cooperative policy-reward optimization for {LLMs}},
  journal = {Expert Syst. Appl.},
  volume  = {301},
  number  = {10},
  year    = {2026},
  note    = {{Art}. no. 130113}
}

@article{Liu2025KBS,
  title   = {Sample-efficient backtrack temporal difference deep reinforcement learning},
  journal = {Knowl.-Based Syst.},
  volume  = {330},
  year    = {2025},
  number  = {},
  issn    = {0950-7051},
  author  = {Qi Liu and Pengbin Chen and Ke Lin and Kaidong Zhao and Jinliang Ding and Yanjie Li},
  note    = {{Art}. no. 114613}
}

@article{Hu_TCNS,
  author  = {Hu, Jiangping and Chen, Bo and Ghosh, Bijoy Kumar},
  journal = {IEEE Trans. Control Netw. Syst.},
  title   = {Formation–Circumnavigation Switching Control of Multiple ODIN Systems via Finite-Time Intermittent Control Strategies},
  year    = {2024},
  volume  = {11},
  number  = {4},
  month   = {Dec.},
  pages   = {1986-1997}
}

@article{GuWang_UAV,
  title   = {A constrained reinforcement learning based approach for cooperative control of multi-{UAV} in dense obstacle environments.},
  author  = {Gu, Jian and Wang, Yin},
  journal = {Sci. China Technological Sci.},
  year    = {2026},
  month   = {Jan.},
  volume  = {69},
  number  = {1},
  pages   = {298-305}
}

@article{Fiedler1973,
  title   = {Algebraic connectivity of graphs},
  author  = {Miroslav Fiedler},
  journal = {Czech. Math. J.},
  year    = {1973},
  volume  = {23},
  number  = {2},
  pages   = {298-305}
}

@article{Li2025TASE,
  author  = {Li, Lulu and Cherouat, Abel and Snoussi, Hichem and Wang, Tian},
  journal = {IEEE Trans. Autom. Sci. Eng.},
  title   = {Grasping With Occlusion-Aware Ally Method in Complex Scenes},
  year    = {2025},
  volume  = {22},
  number  = {},
  pages   = {5944-5954}
}

@article{TAN2026TIE,
  author  = {Tan, Junkai and Xue, Shuangsi and Guan, Qingshu and Guo, Zihang and Cao, Hui and Chen, Badong},
  journal = {IEEE Trans. Ind. Electron.},
  title   = {Fixed-Time Stochastic Learning From Human-{UAV} Interaction With State-Input Constraints},
  year    = {2026},
  volume  = {73},
  number  = {2},
  month   = {Feb.},
  pages   = {3071-3082}
}

@article{TAN2025TASE,
  author  = {Tan, Junkai and Xue, Shuangsi and Li, Huan and Guo, Zihang and Cao, Hui and Chen, Badong},
  journal = {IEEE Trans. Autom. Sci. Eng.},
  title   = {Hierarchical Safe Reinforcement Learning Control for Leader-Follower Systems With Prescribed Performance},
  year    = {2025},
  volume  = {22},
  number  = {},
  month   = {Aug.},
  pages   = {19568-19581}
}

@article{Chen2024RAL,
  author  = {Chen, Yansong and Wu, Yuchen and Yang, Helei and Cao, Junjie and Wang, Qinqin and Liu, Yong},
  journal = {IEEE Robot. Automat. Lett.},
  title   = {A Distributed Pipeline for Collaborative Pursuit in the Target Guarding Problem},
  year    = {2024},
  volume  = {9},
  number  = {3},
  month   = {Mar.},
  pages   = {2064-2071}
}

@article{Gu2026TCCN,
  author  = {Gu, Jian and Wang, Yin and Ji, Wen and Wei, Zhongxiang and Wang, Jingjing},
  journal = {IEEE Trans. Cognit. Commun. Netw.},
  title   = {{LLM}-Based Dynamic Event-Triggered Communication for Multi-{UAV} Formation Control in Urban Environments},
  year    = {2026},
  volume  = {12},
  number  = {},
  pages   = {4825-4838}
}

@article{Yuan2023FITEE,
  author  = {Yuan, Weilin and Chen, Jiaxing and Chen, Shaofei and Feng, Dawei and Hu, Zhenzhen and Li, Peng and Zhao, Weiwei},
  title   = {Transformer in reinforcement learning for decision-making: a survey},
  journal = {Frontiers Inf. Technol. Electron. Eng.},
  year    = {2024},
  volume  = {25},
  number  = {6},
  pages   = {763-790}
}

@article{sun2025IIEM,
  author  = {Sun, G. and L{\"u}, J. and Liu, K. and Wang, Z. and Chen, G.},
  title   = {Mean-shift theory and its applications in swarm robotics: A new way to enhance the efficiency of multi-robot collaboration},
  journal = {IEEE Ind. Electron. Mag.},
  note    = {early access, Nov. 17, 2025, doi: \href{https://doi.org/10.1109/MIE.2025.3627087}{10.1109/MIE.2025.3627087}}
}

@article{ZHENG2025112216,
  author  = {Canlun Zheng and Yize Mi and Hanqing Guo and Huaben Chen and Zhiyun Lin and Shiyu Zhao},
  title   = {Optimal spatial–temporal triangulation for bearing-only cooperative motion estimation},
  journal = {Automatica},
  volume  = {175},
  month   = {May},
  year    = {2025},
  issn    = {0005-1098},
  note    = {{Art}. no. 112216}
}

@article{Lei2024CSL,
  author  = {Lei, Yangqi and Quan, Quan and She, Zhikun},
  journal = {IEEE Control Syst. Lett.},
  title   = {Mean-Field-Based Density Control for Swarm Robotics Passing-Through a Virtual Tube},
  year    = {2024},
  volume  = {8},
  number  = {},
  pages   = {3500-3505}
}

@article{SINIGAGLIA2025112218,
  title   = {Robust optimal density control of robotic swarms},
  journal = {Automatica},
  volume  = {176},
  year    = {2025},
  issn    = {0005-1098},
  author  = {Carlo Sinigaglia and Andrea Manzoni and Francesco Braghin and Spring Berman},
  note    = {{Art.} no.  112218}
}

@article{Thomas2020dilevery,
  author  = {Thayer, Thomas C. and Vougioukas, Stavros and Goldberg, Ken and Carpin, Stefano},
  journal = {IEEE Trans. Autom. Sci. Eng.},
  title   = {Multirobot Routing Algorithms for Robots Operating in Vineyards},
  year    = {2020},
  volume  = {17},
  number  = {3},
  month   = {Jul.},
  pages   = {1184-1194}
}

@article{sun2023mean,
  title     = {Mean-shift exploration in shape assembly of robot swarms},
  author    = {Sun, Guibin and Zhou, Rui and Ma, Zhao and Li, Yongqi and Gro{\ss}, Roderich and Chen, Zhang and Zhao, Shiyu},
  journal   = {Nature Commun.},
  volume    = {14},
  number    = {1},
  month     = {Jun.},
  year      = {2023},
  publisher = {Nature Publishing Group UK London},
  note      = {{Art}. no. 3476}
}

@article{chen2024multipleherding,
  author  = {Chen, Yanjie and Zhang, Zhixing and Wu, Zheng and Wu, Yangning and He, Bingwei and Zhang, Hui and Wang, Yaonan},
  journal = {IEEE Trans. Autom. Sci. Eng.},
  title   = {Multiple Mobile Robots Planning Framework for Herding Non-Cooperative Target},
  year    = {2024},
  volume  = {21},
  number  = {4},
  month   = {Oct.},
  pages   = {7363-7378}
}

@article{Li2025DEFORM,
  author  = {Li, Jin and Xu, Yang and Shi, Xiufang and Li, Liang},
  journal = {IEEE Robot. Automat. Lett.},
  title   = {DEFORM: Adaptive Formation Reconfiguration of Multi-Robot Systems in Confined Environments},
  year    = {2025},
  volume  = {10},
  number  = {5},
  month   = {May},
  pages   = {4706-4713}
}

@article{yang2022autonomous,
  title     = {Autonomous environment-adaptive microrobot swarm navigation enabled by deep learning-based real-time distribution planning},
  author    = {Yang, Lidong and Jiang, Jialin and Gao, Xiaojie and Wang, Qinglong and Dou, Qi and Zhang, Li},
  journal   = {Nature Mach. Intell.},
  volume    = {4},
  number    = {5},
  pages     = {480--493},
  month     = {May},
  year      = {2022},
  publisher = {Nature Publishing Group UK London}
}

@article{Zhao2018Affine,
  author  = {Zhao, Shiyu},
  journal = {IEEE Trans. Autom. Control.},
  title   = {Affine Formation Maneuver Control of Multiagent Systems},
  year    = {2018},
  volume  = {63},
  number  = {12},
  month   = {Dec.},
  pages   = {4140-4155}
}

@article{ZHANG2025111935,
  author  = {Zhang, Xiaozhen and Yang, Qingkai and Xiao, Fan and Chen, Jie},
  title   = {Linear formation control of multi-agent systems},
  journal = {Automatica},
  volume  = {171},
  month   = {Jan.},
  year    = {2025},
  issn    = {0005-1098},
  note    = {{Art.} no. 111935}
}

@article{quan2023robust,
  author  = {Quan, Lun and Yin, Longji and Zhang, Tingrui and Wang, Mingyang and Wang, Ruilin and Zhong, Sheng and Zhou, Xin and Cao, Yanjun and Xu, Chao and Gao, Fei},
  journal = {IEEE Trans. Robot.},
  title   = {Robust and Efficient Trajectory Planning for Formation Flight in Dense Environments},
  year    = {2023},
  volume  = {39},
  number  = {6},
  month   = {Dec.},
  pages   = {4785-4804}
}

@article{huang2024cooperative,
  author  = {Huang, Yunke and Zhang, Shuai},
  journal = {IEEE Robot. Automat. Lett.},
  title   = {Cooperative Object Transport by Two Robots Connected With a Ball-String-Ball Structure},
  year    = {2024},
  volume  = {9},
  number  = {5},
  pages   = {4313-4320},
  month   = {May}
}

@article{Xing2024UAV,
  author  = {Xing, Xiaojun and Zhou, Zhiwei and Li, Yan and Xiao, Bing and Xun, Yilin},
  journal = {IEEE Trans. Veh. Technol.},
  title   = {Multi-{UAV} Adaptive Cooperative Formation Trajectory Planning Based on an Improved {MATD3} Algorithm of Deep Reinforcement Learning},
  year    = {2024},
  volume  = {73},
  number  = {9},
  pages   = {12484-12499}
}

@article{fiorini1998motion,
  title     = {Motion planning in dynamic environments using velocity obstacles},
  author    = {Fiorini, Paolo and Shiller, Zvi},
  journal   = {Int. J. Rob. Res},
  volume    = {17},
  number    = {7},
  pages     = {760--772},
  year      = {1998},
  publisher = {Sage Publications Sage CA: Thousand Oaks, CA}
}

@article{yang2023self,
  title     = {Self-organized polygon formation control based on distributed estimation},
  author    = {Yang, Qingkai and Xiao, Fan and Lyu, Jingshuo and Zhou, Bo and Fang, Hao},
  journal   = {IEEE Trans. Ind. Electron.},
  volume    = {71},
  number    = {2},
  pages     = {1958--1967},
  month     = {Feb.},
  year      = {2024},
  publisher = {IEEE}
}

@article{Zhao2018General,
  author  = {Zhao, Shiyu and Dimarogonas, Dimos V. and Sun, Zhiyong and Bauso, Dario},
  journal = {IEEE Trans. Autom. Control},
  title   = {A General Approach to Coordination Control of Mobile Agents With Motion Constraints},
  year    = {2018},
  volume  = {63},
  number  = {5},
  pages   = {1509-1516},
  month   = {May}
}

@article{Liu2025Observer,
  author  = {Liu, Yangyang and Liu, Chun and Meng, Yizhen and Jiang, Bin and Wang, Xiaofan},
  journal = {IEEE Trans. Autom. Sci. Eng.},
  title   = {Observer-Based Multi-Agent Reinforcement Learning for Pursuit-Evasion Game With Multiple Unknown Uncertainties},
  year    = {2025},
  volume  = {22},
  number  = {},
  month   = {Apr.},
  pages   = {14332-14345}
}

@article{2021Feedback,
  author  = {Selvakumar, Jhanani and Bakolas, Efstathios},
  journal = {IEEE Trans. Cybern.},
  title   = {Feedback Strategies for a Reach-Avoid Game With a Single Evader and Multiple Pursuers},
  year    = {2021},
  volume  = {51},
  number  = {2},
  month   = {Feb.},
  pages   = {696-707}
}

@article{dong2018theory,
  title     = {Theory and experiment on formation-containment control of multiple multirotor unmanned aerial vehicle systems},
  author    = {Dong, Xiwang and Hua, Yongzhao and Zhou, Yan and Ren, Zhang and Zhong, Yisheng},
  journal   = {IEEE Trans. Autom. Sci. Eng.},
  volume    = {16},
  number    = {1},
  pages     = {229--240},
  month     = {Jan.},
  year      = {2019},
  publisher = {IEEE}
}

@article{Unmannedsystem1,
  author          = {Lian, Bosen and Nguyen, Nhan T. and Lewis, Frank L.},
  title           = {Distributed Cluster Containment Control of Swarms via Differential
                     Game-Theoretic Learning},
  journal         = {Unmanned Syst.},
  year            = {2025},
  volume          = {13},
  number          = {05},
  pages           = {1283-1294},
  month           = {Sep.},
  doi             = {10.1142/S2301385025440017},
  earlyaccessdate = {AUG 2025},
  issn            = {2301-3850},
  eissn           = {2301-3869}
}

@article{deng2020multi,
  title     = {Multi-agent cooperative pursuit-defense strategy against one single attacker},
  author    = {Deng, Ziquan and Kong, Zhaodan},
  journal   = {IEEE Robot. Autom. Lett.},
  volume    = {5},
  number    = {4},
  pages     = {5772--5778},
  month     = {Oct.},
  year      = {2020},
  publisher = {IEEE}
}

@article{song2021herding,
  title     = {Herding by caging: a formation-based motion planning framework for guiding mobile agents},
  author    = {Song, Haoran and Varava, Anastasiia and Kravchenko, Oleksandr and Kragic, Danica and Wang, Michael Yu and Pokorny, Florian T and Hang, Kaiyu},
  journal   = {Auton. Robots},
  volume    = {45},
  pages     = {613--631},
  year      = {2021},
  publisher = {Springer}
}

@article{pierson2017controlling,
  title     = {Controlling noncooperative herds with robotic herders},
  author    = {Pierson, Alyssa and Schwager, Mac},
  journal   = {IEEE Trans. Robot.},
  volume    = {34},
  number    = {2},
  pages     = {517--525},
  month     = {Apr.},
  year      = {2018},
  publisher = {IEEE}
}

@article{Liu2024Graph,
  author  = {Liu, Wenhang and Hu, Jiawei and Zhang, Heng and Wang, Michael Yu and Xiong, Zhenhua},
  journal = {IEEE Trans. Robot.},
  title   = {A Novel Graph-Based Motion Planner of Multi-Mobile Robot Systems With Formation and Obstacle Constraints},
  year    = {2024},
  volume  = {40},
  number  = {},
  pages   = {714-728}
}

@article{chipade2021multiagent,
  title     = {Multiagent planning and control for swarm herding in 2-D obstacle environments under bounded inputs},
  author    = {Chipade, Vishnu S and Panagou, Dimitra},
  journal   = {IEEE Trans. Robot.},
  volume    = {37},
  number    = {6},
  pages     = {1956--1972},
  month     = {Dec.},
  year      = {2021},
  publisher = {IEEE}
}

@article{chen2023collision,
  title     = {Collision-Free Formation Control With Global Network Integrity Maintenance via Preserving Induced-Subgraph Connectivity},
  author    = {Chen, Jinyong and Zhou, Rui and Sun, Guibin and Zhang, Jie},
  journal   = {IEEE Syst. J.},
  year      = {2023},
  volume    = {17},
  number    = {3},
  pages     = {4078-4089},
  month     = {Apr.},
  publisher = {IEEE}
}

@book{zhao2024mathematical,
  title     = {Mathematical Foundations of Reinforcement Learning},
  author    = {Zhao, Shiyu},
  year      = {2024},
  publisher = {Springer},
  address   = {Singapore},
  isbn      = {978-981-99-3668-3}
}

\begin{IEEEbiography}[{\includegraphics[width=1in,height=1.25in,clip,keepaspectratio]{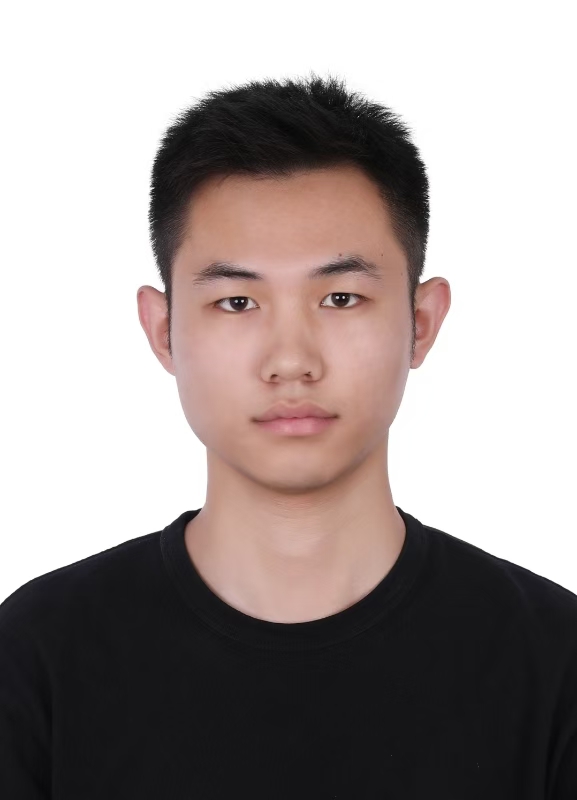}}]{Zhaozong Wang} received the B.S. degree from the University of Science and Technology Beijing, Beijing, China, in 2021, and the M.S. degree from Beihang University, Beijing, in 2023, where he is currently pursuing the Ph.D. degree with the School of Automation Science and Electrical Engineering.

His current research interests include task assignment of multi-agent systems and multi-agent reinforcement learning.
\end{IEEEbiography}

\begin{IEEEbiography}[{\includegraphics[width=1in,height=1.25in,clip,keepaspectratio]{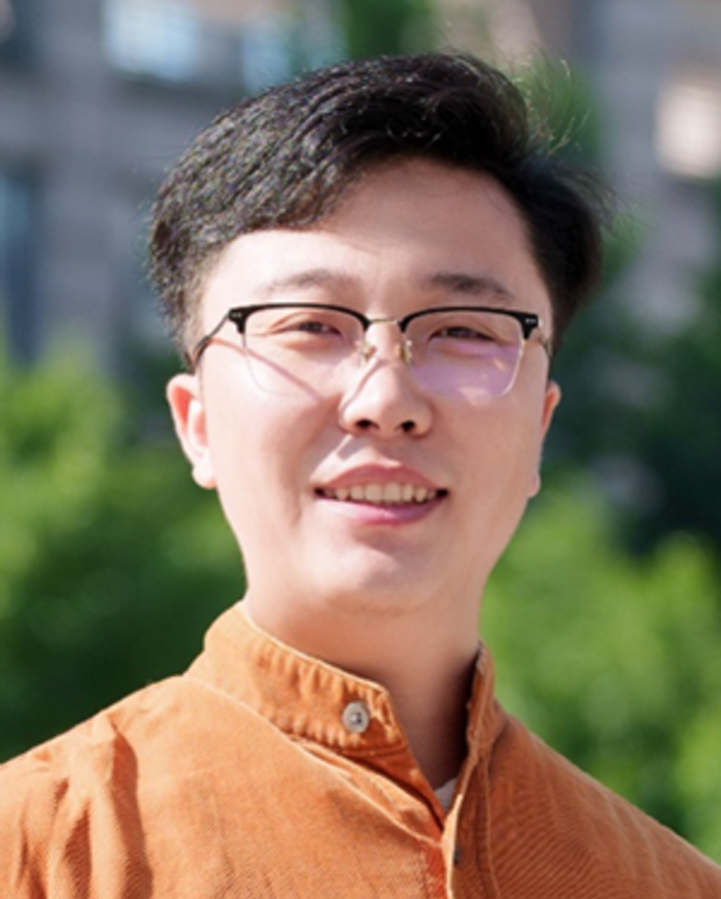}}]{Guibin Sun} received the B.Eng. degree in automation from Inner Mongolia University, China, in 2017, and the Ph.D. degree in control science and engineering from Beihang University, Beijing, China, in 2022. 

From 2022 to 2025, he was a Postdoctoral Research Fellow of “Zhuoyue” Program with Beihang University. From 2021 to 2022, he was a Visiting Student with Westlake University, Hangzhou, China. He is currently an Associate Professor with the School of Automation Science and Electrical Engineering, Beihang University. He has authored more than 20 research papers in international journals, including Nature Communications, IEEE Magazine, and IEEE Transactions. His research interests include theories and applications of robot swarms, especially for large-scale swarms. 

Dr. Sun was the Session Chair for IROS 2024. 
\end{IEEEbiography}

\begin{IEEEbiography}[{\includegraphics[width=1in,height=1.25in,clip,keepaspectratio]{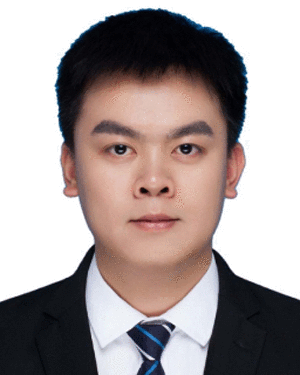}}]{Jinyong Chen} received the B.E. and M.E. degrees from Beihang University, Beijing, China, in 2017 and 2020, respectively. 
He received the Ph.D. degree in control science and engineering from Beihang University, Beijing, China, in 2024. 
His current research interests include coordinated control theory and reinforcement learning, and AIOps. 
\end{IEEEbiography}
\vskip 0pt plus -1fil

\begin{IEEEbiography}[{\includegraphics[width=1in,height=1.25in,clip,keepaspectratio]{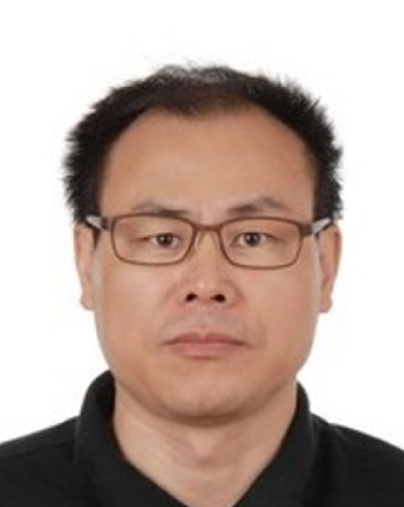}}]{Rui Zhou} received the Ph.D. degree in spacecraft design from Harbin Institute of Technology, Harbin, China, in 1997.

He is currently a Professor and a Ph.D. Supervisor of guidance, navigation and control at Beihang University, Beijing, China. 
His research interests include autonomous control theory, intelligent decision systems, coordinated control and guidance, and information fusion for networked systems.
\end{IEEEbiography}

\end{document}